%% file: main.tex
\pdfoutput=1
\documentclass{article} 
\usepackage{iclr2027_conference,times}
 
\input{math_commands.tex}

\usepackage{microtype}
\usepackage{hyperref}
\usepackage{url}
\usepackage{booktabs}
\usepackage{amsmath}
\usepackage{amsfonts}
\usepackage{amsthm}
\usepackage[T1]{fontenc}
\usepackage{tikz}
\usepackage{algorithm}
\usepackage{algorithmic}
\usepackage{xcolor}
\usepackage{graphicx}        
\usepackage{float}
\usepackage{wrapfig}         
\usepackage{bm}
 
\graphicspath{{images/}}     
 
\newcommand{\fullwidth}{\linewidth}

\usepackage[most]{tcolorbox}
\newtcolorbox{highlight}{
  enhanced, breakable, sharp corners,
  colback=boxblue!40, colframe=arrowblue,
  boxrule=0.6pt, left=3pt, right=3pt, top=2pt, bottom=2pt, boxsep=1pt
}

\usetikzlibrary{positioning, arrows.meta, fit, backgrounds, calc, decorations.pathreplacing}
 \usepackage[capitalize,noabbrev]{cleveref}

\newtheorem{theorem}{Theorem}
\newtheorem{lemma}[theorem]{Lemma}
\newtheorem{proposition}[theorem]{Proposition}
\newtheorem{definition}[theorem]{Definition}
\newtheorem{remark}[theorem]{Remark}
\newtheorem{corollary}[theorem]{Corollary}

\providecommand{\mI}{\bm{I}}
\providecommand{\mZero}{\bm{0}}
\providecommand{\mW}{\bm{W}}
\providecommand{\seqlen}{T}                        
\providecommand{\pos}[1]{\bar{\vp}_{#1}}           
\providecommand{\gph}{\mathcal{G}}                 
\providecommand{\vertices}{\mathcal{V}}            
\providecommand{\edges}{\mathcal{E}}               
\providecommand{\reach}[1]{\vertices_{#1}}         
\providecommand{\rootnode}{r}                            
\providecommand{\cand}[1]{c_{#1}}                  
\providecommand{\ansv}{c_{i^{\star}}}              
\providecommand{\wrongv}{c_{3-i^{\star}}}          
\providecommand{\eg}[1]{u_{#1}}                    
\providecommand{\eh}[1]{v_{#1}}                    
\providecommand{\dist}{\operatorname{dist}}
\providecommand{\ecc}{\operatorname{ecc}}
\providecommand{\walk}[1]{\mathrm{W}_{#1}}         
\providecommand{\diam}{D}
\providecommand{\nmax}{n_{\max}}
\providecommand{\nv}{n}                            
\providecommand{\nedge}{m}                         
 
\providecommand{\tokBOS}{\ensuremath{\langle \mathrm{s}\rangle}}
\providecommand{\tokEdge}{\ensuremath{\langle \mathrm{e}\rangle}}
\providecommand{\tokQ}{\ensuremath{\langle \mathrm{Q}\rangle}}
\providecommand{\tokR}{\ensuremath{\langle \mathrm{R}\rangle}}
\providecommand{\tokA}{\ensuremath{\langle \mathrm{A}\rangle}}
 
\providecommand{\emb}[1]{\bm{\phi}_{#1}}           
\providecommand{\embmat}{\bm{\Phi}}                
\providecommand{\vocab}{\mathbb{V}}
\providecommand{\dte}{d_{\mathrm{te}}}
\providecommand{\dpe}{d_{\mathrm{pe}}}
 
\providecommand{\ct}{\operatorname{ct}}            
\providecommand{\bufa}{\operatorname{b}_1}         
\providecommand{\bufb}{\operatorname{b}_2}         
\providecommand{\bufc}{\operatorname{b}_3}         
\providecommand{\bufd}{\operatorname{b}_4}         
\providecommand{\pe}{\operatorname{pe}}            
\providecommand{\pidx}{\operatorname{Pos}}         
 
\providecommand{\Prop}{\bm{A}_{\gph}}              
\providecommand{\chsr}[2]{\mathrm{Ch}(#1,#2)}      
\providecommand{\gain}{\lambda}                    
\providecommand{\mgn}{\delta}                      
\providecommand{\nsteps}{N}                        
\providecommand{\prm}{c}                           
\providecommand{\dnz}{\widehat{\vx}_{\theta}}      
\providecommand{\vfield}{\bm{v}_{\theta}}          
\providecommand{\nse}{\bm{\xi}}                    

 \providecommand{\Ntrain}{N_{\mathrm{train}}}
\providecommand{\Nback}{N_{\mathrm{back}}}
\providecommand{\Neval}{N_{\mathrm{eval}}}
\providecommand{\Wdec}{\bm{W}_{\mathrm{dec}}}   

\providecommand{\Tmax}{T_{\max}}
\providecommand{\epe}{\varepsilon_{\mathrm{pe}}}
\providecommand{\epsth}{\varepsilon_{\mathrm{th}}}   

\definecolor{TokenPurple}{HTML}{E8D7F0}
\definecolor{TokenBlue}{HTML}{D8EDF8}
\definecolor{TokenOrange}{HTML}{FAE4D4}
\definecolor{TokenGreen}{HTML}{DDF3DE}
\definecolor{ArrowDarkBlue}{HTML}{2B5C7D}
 
\definecolor{darkblue}{rgb}{0, 0, 0.5}
\hypersetup{colorlinks=true, citecolor=darkblue, linkcolor=darkblue, urlcolor=darkblue}
 
\definecolor{tokencolor}{RGB}{236, 218, 238}   
\definecolor{nodecolor}{RGB}{214, 236, 245}    
\definecolor{matrixcolor}{RGB}{245, 226, 214}  
\definecolor{answercolor}{RGB}{214, 240, 214}  
\definecolor{arrowcolor}{RGB}{197, 146, 117}   
\definecolor{arrowblue}{RGB}{43, 85, 115}
\definecolor{boxpink}{RGB}{242, 224, 242}
\definecolor{boxblue}{RGB}{221, 239, 250}
\definecolor{boxorange}{RGB}{251, 237, 226}
\definecolor{boxgreen}{RGB}{223, 244, 223}
\definecolor{textgray}{RGB}{50, 50, 50}
 
\title{Unrolled Flow Models for Reasoning}
 
\author{%
    Faissal Izermine\textsuperscript{\textnormal{1,2}},
    Hanru Bai\textsuperscript{\textnormal{1,3}},
    Oscar Davis\textsuperscript{\textnormal{4}},
    T. Konstantin Rusch\textsuperscript{\textnormal{1,2,5,6}\thanks{\scriptsize Correspondence to \href{mailto:tkrusch@tue.ellis.eu}{\texttt{tkrusch@tue.ellis.eu}}}~}\\
    {\small \textsuperscript{1}Max Planck Institute for Intelligent Systems,}
    {\small \textsuperscript{2}ELLIS Institute Tübingen,} 
    {\small \textsuperscript{3}ETH Zurich,}\\
    {\small \textsuperscript{4}University of Oxford,} 
    {\small \textsuperscript{5}Tübingen AI Center,}
    {\small \textsuperscript{6}Liquid AI}
}

\iclrfinalcopy 
 
\begin{document}
 
\maketitle

\begin{abstract}
Flow matching enables language generation in few steps, but whether
additional integration steps improve reasoning remains unclear. We prove that a flow parameterized by a two-layer Transformer can solve graph reachability, with the required number of integration steps increasing with the target’s distance from the root. Yet, standard flow language models can fail to benefit from additional steps on reasoning tasks. We attribute this limitation to objectives that supervise each time point independently, without explicitly training successive steps to build on one another. To address this, we instead train through the model’s own latent rollout over a randomly sampled subinterval of $[0, 1]$, decoding only at the endpoint. On ProsQA, this raises accuracy to $97\%$ and enables performance to improve with additional integration steps. For the longer rollouts required by reasoning tasks such as Sudoku and Maze, retracting the latent state onto a sphere stabilizes the dynamics and yields substantial gains over baselines with more than three times as many parameters. Sampling multiple rollouts further improves performance when paired with a parameter-free selection score, although reliable selection remains challenging for longer answers. Together, these results establish a theoretical basis for reasoning with flows and show how rollout training, stable latent dynamics, and rollout selection help realize this capacity in practice. The implementation can be found \href{https://github.com/camail-official/Unrolled_Flow_Models_for_Reasoning_UFM.git}{here}.
\end{abstract}


\section{Introduction}

Autoregressive language models can spend more computation at inference by generating a longer chain of thought \citep{wei2022cot} or, more broadly, by scaling test-time compute \citep{snell2024testtime}. Flow matching \citep{Lipman2023FlowMatching, Liu2023RectifiedFlow, Albergo2023StochasticInterpolants} offers a different way to produce discrete answers \citep{eijkelboom2024variational}. Instead of predicting one token after
another, flow models start from noise and repeatedly update a continuous state that is decoded into a token sequence. The number of updates can be chosen at inference time. This makes flow models an appealing candidate for adaptive test-time computation.

Additional integration steps do not inherently guarantee deeper reasoning. In generation tasks, a small number of steps is often enough, and recent work has pushed flow language models toward one- or few-step sampling by learning \emph{flow maps} that make direct jumps along a trajectory \citep{FlowMapLM2026, CategoricalFlowMaps2026, ScalingCategoricalFlowMaps2026, SelfConditionedFlowMap2026}. For reasoning, additional steps are useful only if later updates can build
on information produced by earlier ones.

Prior work in latent and recursive reasoning motivates this work. Coconut shows that autoregressive language models can update a continuous latent state instead of emitting every intermediate thought as text \citep{Hao2024Coconut}. On graph reachability, \citet{zhu2025reasoning} show that repeated latent updates can maintain a superposition of reachable vertices and expand it one hop at a time. Recursive reasoners similarly solve structured tasks by applying a small, weight-tied network repeatedly to an internal state \citep{wang2025hrm, JolicoeurMartineau2025}. An Euler rollout of a flow has the same basic form: one network is applied again and again to a latent state. We investigate how these repeated applications cooperate.

To show what integration steps can compute, we construct explicit weights for
a two-layer Transformer that defines a flow solving the graph reachability problem. Its Euler rollout reaches vertices one hop per step, so the number
of steps needed to separate the candidates depends on the distance from the root to the answer. However, this existence result does not establish that training recovers the construction (\Cref{sec:theory}).

Standard training of flow models does not directly assess whether an update produces
a state that later updates can leverage. To address this limitation, we propose \emph{Unrolled Flow Model} (UFM), which trains through its own latent rollout using a single prediction loss at the final state. To reduce memory usage for long rollouts, we backpropagate this loss through only the last few steps (\Cref{sec:method}). We stabilize these rollouts using sphere retraction to control latent-state growth (\Cref{sec:scaling}). On ProsQA and Sudoku-Hard, UFM improves with additional steps, whereas evaluated flow baselines show little gain. Multiple rollouts further improve reliability, although selecting among them remains challenging for tasks with long answers, such as maze solving.

Our contributions are: (i) We give a constructive result
showing how flow integration can implement graph search with the number of steps
determined by the instance. (ii) We introduce training through a terminal
rollout loss and show that it makes flow steps useful for reasoning in our
experiments, while connecting the resulting dynamics to a superposition-based
mechanism. (iii) We use sphere retraction to stabilize long rollouts and demonstrate stronger scaling with additional steps than substantially larger flow baselines on structured reasoning tasks.

\section{Background}
\label{sec:background}

\subsection{Conditional flow matching}

\paragraph{Setup.}
Given a prompt $c$, the model generates an answer with $L$ positions.
A flow transforms noise at $t=0$ into an answer representation at $t=1$,
using $N$ Euler steps on a grid $0=t_0<\cdots<t_N=1$.

\paragraph{Flow matching.}
Flow matching \citep{Lipman2023FlowMatching, Liu2023RectifiedFlow, Albergo2023StochasticInterpolants} defines a path between noise $\vx_0 \sim p_0 = \mathcal{N}(0,\mI)$ and data $\vx_1 \sim p_1$. We use the linear interpolant
\begin{equation}
\label{eq:interpolant}
I_t := (1-t)\,\vx_0 + t\,\vx_1,
\qquad
\dnz(\vx, t) \approx \mathbb{E}\big[\vx_1 \mid I_t = \vx\big].
\end{equation}
The denoiser $\dnz$ predicts the clean endpoint from a point on this path. Since $\vx_1 - \vx_0 = (\vx_1 - I_t)/(1-t)$, its corresponding velocity field is
\begin{equation}
\label{eq:velocity}
\vfield(\vx, t) = \frac{\dnz(\vx, t) - \vx}{1-t}.
\end{equation}
Integrating $\dot{\vz}_t = \vfield(\vz_t,t)$ from $t=0$ to $t=1$ transports noise toward the data distribution. In practice, we approximate this integration with Euler steps:
\begin{equation}
\label{eq:euler}
\vz_{t_{k+1}} = \vz_{t_k} + (t_{k+1} - t_k)\,\vfield(\vz_{t_k}, t_k).
\end{equation}

\paragraph{Conditioning and training.}

The prompt tokens are concatenated with $L$ answer slots and processed by a bidirectional Diffusion Transformer (DiT) \citep{Vaswani2017Attention, Peebles2023DiT}. The DiT receives the flow time $t$ as a conditioning input, while the prompt remains clean throughout the rollout; only the answer slots evolve.  The conditional denoiser is trained using:

\begin{equation}
\label{eq:fm-loss}
\mathcal{L}_{\mathrm{FM}}(\theta) =
\mathbb{E}_{(c,\vx_1),\,\vx_0,\,t \sim \mathcal{U}[0,1]}
\Big[\ell\big(\dnz(I_t, t, c),\ \vx_1\big)\Big].
\end{equation}

Here, $\ell$ is the per-example loss. Thus, each training example supervises a denoising prediction at one sampled time; the model is never evaluated on a state produced by its own preceding update. We return to this distinction in
\Cref{sec:failure}.

\subsection{Flow Language Models}
Flow Language Models (FLMs) apply this formulation to discrete tokens. FLM \citep{FlowMapLM2026} represents the clean answer as a stack of one-hot vectors, $\vx_1 \in \{0,1\}^{L \times V}$. Its denoiser therefore predicts a distribution over the vocabulary at every answer position:
\begin{equation}
\dnz(\vx, t, c) =
\big(p^{i}(\cdot \mid I_t = \vx,\ c)\big)_{1 \leq i \leq L}.
\end{equation}
The model is trained with cross-entropy, and the evolving state $\vz_t \in \mathbb{R}^{L \times V}$ is moved toward this predicted distribution at each Euler step. In our experiments, FLM refers to the multi-step flow model, rather than the distilled one-step flow-map variant introduced in the same work.

\paragraph{S-FLM.}
S-FLM \citep{HypersphericalFlows2026} instead represents tokens with learned unit-norm embeddings on the hypersphere. It follows a geodesic path from noise toward these embeddings, rather than evolving a vocabulary-sized state as in FLM. Its denoiser still produces a distribution over the vocabulary; each sampling step follows a posterior-weighted average of geodesic directions toward the token embeddings. Both methods train a cross-entropy denoiser at one sampled point on
their respective noise paths. S-FLM additionally uses a truncated, adaptive noise schedule. ELF \citep{ELF2026} also defines the flow in a continuous token-embedding space.

FLM and S-FLM are our flow baselines. UFM also evolves a continuous latent state, but it predicts a latent target and applies the vocabulary decoder only once, at the endpoint. Its rollout is trained jointly with a terminal loss, as described in \Cref{sec:method}.

\section{Integration Steps and Reasoning Depth}

\subsection{A Two-Layer Flow for Graph Reachability}
\label{sec:theory}

We define the graph-reachability task and the flow construction used in the existence result. The theorem below shows that the construction solves the task; its Euler discretization will later show how many integration steps are needed. Notation is collected in \Cref{tab:notation}.

\paragraph{Graph and prompt.}
Let $\gph=(\vertices,\edges)$ be a directed graph with $\nv=|\vertices|\le\nmax$ vertices and $\nedge=|\edges|$ edges. Let $\rootnode\in\vertices$ be a designated starting vertex, called the
root, and let $\cand1,\cand2$ be two candidate destinations. Exactly one candidate, denoted $\ansv$, is reachable from $\rootnode$; we write $\wrongv$ for the other candidate.

We write $\dist(\rootnode,v)$ for the length of a shortest directed path from $\rootnode$ to $v$, and $\reach{k}:=\{v\in\vertices:\dist(\rootnode,v)\le k\}$  for the vertices reachable within $k$ hops. The sequence
$\reach{0}\subseteq\reach{1}\subseteq\cdots$ stops growing at the eccentricity of the root, $\diam:=\max\{\dist(\rootnode,v):v\text{ is reachable from }\rootnode\}.$ 
Thus, $\dist(\rootnode,\ansv)\le\diam$.

We also write
$\walk{k}(\rootnode,v)$ for the number of directed walks of length $k$ from $\rootnode$ to $v$.

Following \citet{zhu2025reasoning}, the prompt lists the directed edges and then asks which candidate is reachable (\Cref{fig:prompt} in the appendix):
\begin{equation*}
  \prm=(\tokBOS,\eg1,\eh1,\tokEdge,\ldots,\eg\nedge,\eh\nedge,
        \tokEdge,\tokQ,\cand1,\cand2,\tokR,\rootnode,\tokA).
\end{equation*}
The flow state $\vz_t$ evolves at the position marked by $\tokA$, where the final answer is decoded.

\paragraph{Representation.}
Let $V$ be the vocabulary size and $\dte$ the embedding dimension. Each token $v$ has an embedding $\emb{v}\in\mathbb{R}^{\dte}$. As in \citet{zhu2025reasoning}, we assume orthonormal token embeddings,
\[
  \embmat^\top\embmat=\mI,
  \qquad
  \embmat=[\emb{1},\ldots,\emb{V}],
  \qquad
  \dte\ge V.
\]

This lets $\langle\emb{v},\vz_t\rangle$ recover the coefficient of vertex $v$ in the state. This assumption applies solely to the construction; the trained models use standard embeddings.

The first layer stores each edge's source and target embeddings in separate buffers at its $\tokEdge$ position, and the two candidate embeddings in buffers at $\tokA$. It initializes the answer state as $\vz_0=\emb{\tokA}+\emb{\rootnode}$ (\Cref{lem:chooser,lem:cache}).

\paragraph{Constructed velocity field.}
The second layer is a softmax-free linear-attention head
(\Cref{def:prop}). It reads the content slot of $\tokA$ as its query and the cached source and target edge embeddings as its keys and values. Using these cached edge endpoints, it applies $\Prop=\sum_{(\eg j\to\eh j)\in\edges}\emb{\eh j}\emb{\eg j}^{\top}$, the adjacency operator of $\gph$ in embedding coordinates. With a fixed gain $\gain>0$, the head induces

\begin{equation}
\label{eq:construction-ode}
  \dot{\vz}_t=\gain\Prop\vz_t,
  \qquad
  \vz_0=\emb{\tokA}+\emb{\rootnode}.
\end{equation}

We use a linear head where \citet{zhu2025reasoning} use softmax attention because it makes the velocity linear in the state, so the flow has a closed form. At time $t=1$, the construction compares
$\langle\emb{\cand1},\vz_1\rangle$ and
$\langle\emb{\cand2},\vz_1\rangle$ and returns the candidate with the larger score. This is a candidate-restricted readout. The experiments use the stricter full-vocabulary decoder; we discuss that difference in \Cref{rem:decoding}.

\paragraph{Exact flow}

The constructed velocity field has a closed-form solution. This gives the existence result: the flow separates the reachable candidate from the unreachable one on every instance.

\begin{figure}[!t]
  \centering
  \resizebox{0.9\textwidth}{!}{%
    \begin{tikzpicture}[
        scale=0.8, transform shape,
        font=\sffamily\small,
        tokenbox/.style={
          rectangle, rounded corners=1.5pt, fill=nodecolor, draw=nodecolor!80!black, thick,
          minimum width=1.6cm, minimum height=0.6cm, align=center, font=\large
        },
        promptbox/.style={
          rectangle, rounded corners=2pt, fill=matrixcolor, draw=matrixcolor!80!black, thick,
          minimum width=1.6cm, minimum height=0.5cm, align=center, font=\large
        },
        arrow/.style={-{Stealth[length=2.5mm, width=2.5mm]}, thick, draw=arrowcolor},
        diffarrow/.style={-{Stealth[length=2.5mm, width=2.5mm]}, ultra thick, draw=arrowcolor}
      ]
 
      \node[rectangle, rounded corners=2pt, fill=tokencolor, draw=tokencolor!80!black, thick,
            minimum height=3cm, minimum width=1.4cm, align=center]
        (noise) at (-3, 0.4) {Gaussian\\Noise\\[2mm] $\mathcal{N}(0, \mI_d)$};
 
      \begin{scope}[shift={(0,0)}]
        \node[promptbox]                        (prompt0) at (0,  2.5) {Prompt $c$};
        \node[tokenbox]                         (z0_1)    at (0,  1.4) {$z_0^{(1)}$};
        \node[tokenbox]                         (z0_2)    at (0,  0.6) {$z_0^{(2)}$};
        \node[font=\normalsize, text=gray!80]   (z0_dots) at (0,  0)   {$\vdots$};
        \node[tokenbox]                         (z0_T)    at (0, -0.6) {$z_0^{(L)}$};
        \node[below=0.2cm of z0_T, align=center, font=\large]
          {Time $0$\\[0.5mm] $z_0 \in \mathbb{R}^{L \times d}$};
 
        \draw[arrow] (prompt0) -- (z0_1);
      \end{scope}
 
      \begin{scope}[shift={(3.5,0)}]
        \node[promptbox]                        (promptt) at (0,  2.5) {Prompt $c$};
        \node[tokenbox]                         (zt_1)    at (0,  1.4) {$z_t^{(1)}$};
        \node[tokenbox]                         (zt_2)    at (0,  0.6) {$z_t^{(2)}$};
        \node[font=\normalsize, text=gray!80]   (zt_dots) at (0,  0)   {$\vdots$};
        \node[tokenbox]                         (zt_T)    at (0, -0.6) {$z_t^{(L)}$};
        \node[below=0.2cm of zt_T, align=center, font=\large]
          {Time $t$\\[0.5mm] $z_t \in \mathbb{R}^{L \times d}$};
 
        \draw[arrow] (promptt) -- (zt_1);
      \end{scope}
 
      \begin{scope}[shift={(7,0)}]
        \node[promptbox]                        (promptT) at (0,  2.5) {Prompt $c$};
        \node[tokenbox]                         (zT_1)    at (0,  1.4) {$z_1^{(1)}$};
        \node[tokenbox]                         (zT_2)    at (0,  0.6) {$z_1^{(2)}$};
        \node[font=\normalsize, text=gray!80]   (zT_dots) at (0,  0)   {$\vdots$};
        \node[tokenbox]                         (zT_T)    at (0, -0.6) {$z_1^{(L)}$};
        \node[below=0.2cm of zT_T, align=center, font=\large]
        {Time $1$\\[0.5mm] $z_1 \in \mathbb{R}^{L \times d}$}; 
        \draw[arrow] (promptT) -- (zT_1);
      \end{scope}
 
      \node[align=center, inner sep=0pt] (dec_text) at (10.5, 0.4)
        {\small Decode\\[0.5mm] \small $d \to V$};
      \begin{scope}[on background layer]
        \filldraw[fill=matrixcolor, draw=arrowcolor, thick, rounded corners=1pt]
          (dec_text.center) ++(-0.9, 1.3) -- ++(1.8, 1.0) -- ++(0, -4.6) -- ++(-1.8, 1.0) -- cycle;
        \coordinate (dec_left)  at ($(dec_text.center) + (-0.9, 0)$);
        \coordinate (dec_right) at ($(dec_text.center) + ( 0.9, 0)$);
      \end{scope}
 
      \node[rectangle, rounded corners=2pt, fill=nodecolor, draw=nodecolor!80!black, thick,
            minimum height=4.6cm, minimum width=0.6cm, align=center, anchor=west]
        (softmax) at (dec_right) {\rotatebox{90}{Softmax}};
 
      \node[rectangle, rounded corners=2pt, fill=tokencolor, draw=tokencolor!80!black, thick,
            minimum height=4.6cm, minimum width=1.4cm, align=center, right=0.8cm of softmax]
        (vocab) {Vocab\\Space\\[2mm] $\mathbb{R}^V$};
 
      \draw[arrow] (noise.east) to[out=0, in=180] (z0_1.west);
      \draw[arrow] (noise.east) to[out=0, in=180] (z0_2.west);
      \draw[arrow] (noise.east) to[out=0, in=180] (z0_T.west);
 
      \begin{scope}[shift={(1.1, 0)}]
        \draw[arrow, draw=arrowcolor!80, thin] (0,  1.2) to[out= 10, in=170] (1.3,  1.1);
        \draw[arrow, draw=arrowcolor]          (0,  0.6) to[out=-15, in=195]
          node[above, font=\sffamily\scriptsize, text=arrowcolor, yshift=1mm] {Latent Flow} (1.3, 0.6);
        \draw[arrow, draw=arrowcolor!80, thin] (0,  0.0) to[out= 15, in=165] (1.3,  0.1);
        \draw[arrow, draw=arrowcolor!80, thin] (0, -0.6) to[out=-10, in=190] (1.3, -0.5);
      \end{scope}
 
      \begin{scope}[shift={(4.6, 0)}]
        \draw[arrow, draw=arrowcolor!80, thin] (0,  1.1) to[out= 15, in=165] (1.3,  1.2);
        \draw[arrow, draw=arrowcolor]          (0,  0.6) to[out=-10, in=190]
          node[above, font=\sffamily\scriptsize, text=arrowcolor, yshift=1mm] {Latent Flow} (1.3, 0.6);
        \draw[arrow, draw=arrowcolor!80, thin] (0,  0.1) to[out= 20, in=160] (1.3,  0.0);
        \draw[arrow, draw=arrowcolor!80, thin] (0, -0.5) to[out=-15, in=195] (1.3, -0.6);
      \end{scope}
 
      \draw[arrow] (8.1, 0.4) -- (dec_left);
      \draw[arrow] (softmax.east) -- (vocab.west);
 
    \end{tikzpicture}%
  }
  \caption{\textbf{UFM computes in latent space and decodes once.} At inference, a DiT repeatedly updates a noisy latent answer state from $t=0$ to $t=1$. A linear vocabulary decoder is applied only to the final state.}
  \label{fig:pipeline}
\end{figure}
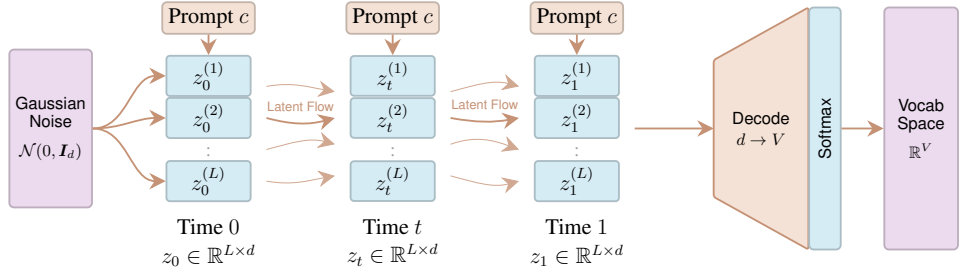

\begin{highlight}
\begin{theorem}[Existence]
\label{thm:existence}

Fix a vocabulary of size $V$, $\nmax$, and $\gain>0$.
There exists a two-layer Transformer, consisting of a
prompt-caching layer followed by a linear-attention layer,
whose weights are independent of the graph, root, and
candidate pair. Its answer state satisfies
\Cref{eq:construction-ode} for every directed graph with $|\vertices|\le\nmax$, every root $\rootnode$, and every candidate pair
with exactly one reachable candidate $\ansv$. At time $t=1$,
\[
  \langle\emb{v},\vz_1\rangle
  =
  \sum_{k\ge0}\frac{\gain^k}{k!}\,\walk{k}(\rootnode,v),
  \qquad v\in\vertices.
\]
This coefficient is positive exactly when $v$ is reachable from
$\rootnode$. The candidate-restricted readout therefore returns $\ansv$,
with the margin bound
\[
  \mgn:=\min_{0\le k\le\nmax-1}\frac{\gain^k}{k!}
  \le
  \langle\emb{\ansv},\vz_1\rangle
  -\langle\emb{\wrongv},\vz_1\rangle.
\]
\end{theorem}
\end{highlight}

\paragraph{Why the exact flow solves reachability.}
Since $\Prop\emb{\tokA}=\mZero$, the solution is
$\vz_t=\emb{\tokA}+e^{\gain t\Prop}\emb{\rootnode}$.
Orthonormality makes $\Prop^k\emb{\rootnode}$ count length-$k$ walks
(\Cref{lem:walks}), giving the theorem's nonnegative coefficient series.
An unreachable vertex has no such walks; a reachable vertex has a
shortest path of length $k^\star\le\nmax-1$, contributing at least
$\gain^{k^\star}/k^\star!\ge\mgn$ at $t=1$.
The network realization is established in \Cref{lem:cache,def:prop}.

Unlike the Euler rollout considered next, the exact flow assigns positive mass to every reachable vertex at any $t>0$. Reaching vertices one hop per step is therefore a property of the discretization, not of the flow itself.

\begin{corollary}[Euler steps match graph distance]
\label{cor:euler}
For any Euler grid $0=t_0<\cdots<t_N=1$ with positive step sizes,
\[
  \vz_{t_N}
  =
  \emb{\tokA}
  +
  \prod_{i=0}^{N-1}
  \bigl(\mI+\gain\Delta t_i\Prop\bigr)\emb{\rootnode}.
\]
Hence, for every vertex $v$,
\[
  \langle\emb{v},\vz_{t_N}\rangle>0
  \quad\Longleftrightarrow\quad
  v\in\reach{N}.
\]

Thus, the candidate-restricted readout is correct when
$N\ge\dist(\rootnode,\ansv)$. Taking
$N\ge\diam=\ecc(\rootnode)$ suffices for every valid candidate
pair with this root.

\end{corollary}

Expanding the product in \Cref{cor:euler} gives positive coefficients
for $\Prop^0,\ldots,\Prop^N$. The state therefore contains contributions
from walks of length at most $N$, and its vertex support is exactly
$\reach{N}$. Step sizes change the weights on these walks and the number
of steps determines how far information can propagate: a target farther from the root requires more Euler steps, using the same weights.

\begin{remark}[Noisy initialization]
\label{rem:gaussian-main}
For $\vz_0=\emb{\tokA}+\emb{\rootnode}+\sigma\nse$, with
$\nse\sim\mathcal{N}(0,\mI)$, the exact time-$1$ flow returns the
correct candidate with probability at least
\[
  1-\exp\!\left(
  -\frac{\mgn^2}{4\sigma^2e^{2\gain\sqrt{\nedge}}}
  \right).
\]
The bound weakens with noise and graph size; see \Cref{app:gaussian}.
\end{remark}

The construction uses orthonormal embeddings, a softmax-free propagation
head, a prompt cache, and a time-independent velocity field.
It does not show that gradient descent recovers this solution. Trained
UFM uses learned embeddings and reapplies its DiT at every step;
\Cref{sec:superposition} examines whether its dynamics show a related
mechanism.


\subsection{Integration Steps Alone Do Not Produce Reasoning}
\label{sec:failure}

We test whether increasing the number of integration steps at
inference improves reasoning accuracy in standard flow models.

\paragraph{Setup.}

\begin{wrapfigure}[26]{r}{0.54\columnwidth}
  \centering
  \includegraphics[width=\linewidth]{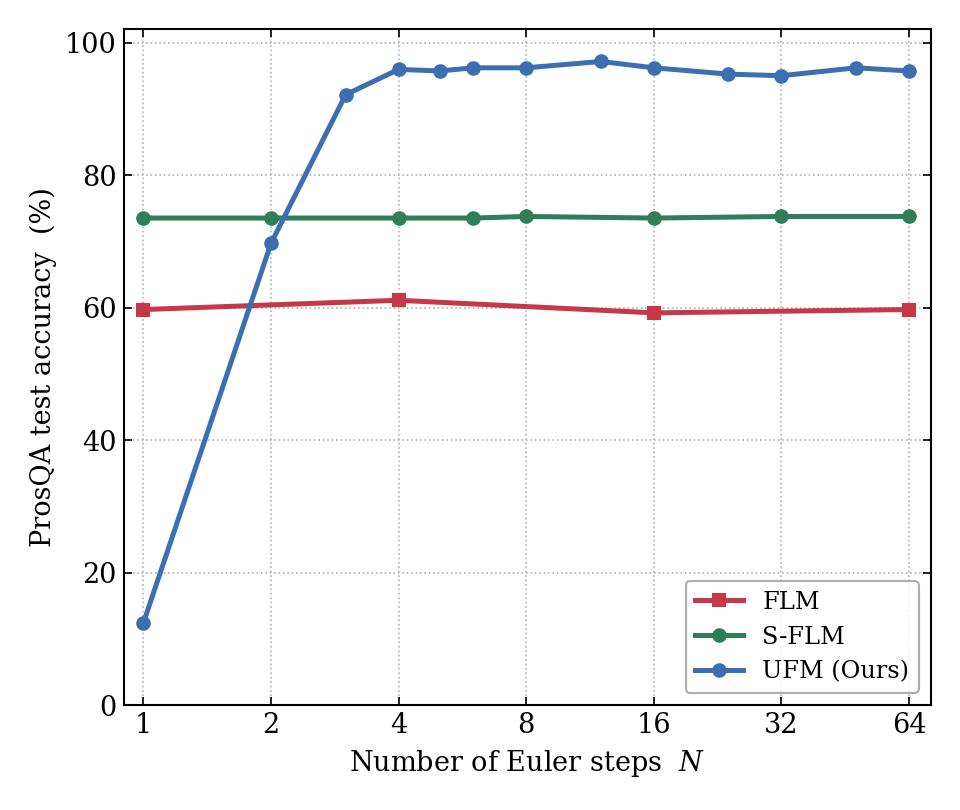}
 \caption{\textbf{Rollout training makes additional steps useful.}
ProsQA test accuracy versus the number of Euler steps at inference.
FLM and S-FLM remain nearly flat, while a single UFM trained with
$\Ntrain=5$ improves from $12\%$ at $N=1$ to approximately $97\%$
with longer rollouts. Decoding uses the full vocabulary.}
  \label{fig:teaser}
\end{wrapfigure}

As in \citet{zhu2025reasoning}, we use ProsQA \citep{Hao2024Coconut}, a directed graph-reachability task
with two candidate answers. The models have two layers and approximately
$15.7$M parameters. The plotted UFM is trained with
$\Ntrain=5$ steps and gradients through all five updates. UFM and S-FLM have approximately
matched training budgets measured by the total number of model evaluations
across examples; configurations are given in \Cref{app:exp}.

\Cref{fig:teaser} shows test accuracy as the number of Euler steps $N$ increases at inference. Both baseline curves are nearly flat: more integration steps add little accuracy. We revisit the same baselines on the harder Sudoku and Maze benchmarks in
\Cref{sec:flow-comparison}. On Sudoku-Hard, the $28.6$M baselines change by
only about one point from $N=8$ to $N=128$ (\Cref{tab:acc-vs-n}). Consequently, standard flow models require more than additional forward passes to improve reasoning.

This behavior is consistent with the standard flow objective in
\Cref{eq:fm-loss}. FLM and S-FLM train the denoiser to recover the answer
from a ground-truth interpolant at a sampled time. The model is therefore
not trained on the states created by its own preceding updates, and the
objective does not assess whether one update leaves a useful state for the
next. UFM instead trains a sequence of updates against one terminal answer
loss.

S-FLM already evolves an embedding-space state; yet, its accuracy changes
little with additional steps. Our construction motivates keeping the
computation in latent space and decoding only at the endpoint;
\Cref{sec:method} trains these updates jointly through the rollout.


\section{Unrolled Flow Model}
\label{sec:method}

We want early updates to produce states that later updates can use to reach the answer. We therefore train an \emph{Unrolled Flow Model} (UFM) 
through its own rollout with a loss on the final answer. UFM keeps the answer state in a continuous latent
space and decodes only once, at the endpoint.

\paragraph{Latent rollout.}
$\vz_t\in\mathbb{R}^{L\times d}$, where $d$ is the hidden width. A DiT predicts a latent target, $\dnz(\vz_t,t,c)\in\mathbb{R}^{L\times d}$, and defines the velocity as in \Cref{eq:velocity}. 


Unlike FLM and S-FLM, intermediate states are not decoded into distributions over the vocabulary. The vocabulary decoder is applied only to the final state,
\[
  \hat{\vy}
  =
  \operatorname{softmax}\!\bigl(\vz_{t_N}\Wdec^\top\bigr),
  \qquad
  \Wdec\in\mathbb{R}^{V\times d}.
\]

We use softmax on ProsQA and replace it with StableMax \citep{prieto2025stablemax} on the longer Sudoku and
Maze tasks. Details are in \Cref{app:long-rollout}.

\paragraph{Terminal rollout loss.}

We apply the loss only at the endpoint, so backpropagation
trains successive updates through their combined effect on
the final answer.

During training, the rollout starts from a partially noised answer at a sampled time $t_0$, $\vz_{t_0}=(1-t_0)\nse+t_0\operatorname{Embed}(\vy)$ with $\nse\sim\mathcal{N}(0,\mI)$, and is integrated over a sampled subinterval ending at $t_N=1$.

\begin{equation}
\label{eq:terminal-loss}
  \mathcal{L}_{\mathrm{UFM}}(\theta)
  =
  \mathbb{E}_{(c,\vy),\,\nse,\,t_0}
  \left[
    \operatorname{CE}\!\left(
      \operatorname{softmax}(\vz_{t_N}\Wdec^\top),\,
      \vy
    \right)
  \right].
\end{equation}
No intermediate state receives a vocabulary loss. At inference, we set $t_0=0$, start from noise, and integrate the full interval $[0,1]$. The complete long-rollout training and inference procedures, including sphere
retraction and truncated backpropagation, are given in \Cref{app:long-rollout}.

\begin{algorithm}[t]
  \caption{Base UFM training on one batch. Task-specific variants
  are described in the text.}
  \label{alg:latent_velocity_matching}
  \begin{algorithmic}[1]
    \REQUIRE Batch $(c,\vy)$, parameters $\theta$,
    rollout length $\Ntrain$,
    backpropagation length $1\le\Nback\le\Ntrain$

    \STATE Sample noise $\nse\sim\mathcal{N}(0,\mI)$
    in the latent answer space
    \STATE Sample a task-specific grid
    $0\le t_0<t_1<\cdots<t_{\Ntrain}=1$
    \STATE $\vz_{t_0}\gets
    (1-t_0)\nse+t_0\operatorname{Embed}(\vy)$

    \FOR{$k=0$ \TO $\Ntrain-1$}
      \STATE $\Delta t_k\gets t_{k+1}-t_k$
      \STATE $\dnz^{(k)}\gets\dnz(\vz_{t_k},t_k,c)$
      \STATE $\vfield^{(k)}\gets
      (\dnz^{(k)}-\vz_{t_k})/(1-t_k)$
      \STATE $\vz_{t_{k+1}}\gets
      \vz_{t_k}+\Delta t_k\,\vfield^{(k)}$
      \IF{$k<\Ntrain-\Nback$}
        \STATE $\vz_{t_{k+1}}\gets
        \operatorname{sg}(\vz_{t_{k+1}})$
        \COMMENT{stop gradients through early updates}
      \ENDIF
    \ENDFOR

    \STATE $\hat{\vy}\gets
    \operatorname{softmax}(\vz_{t_{\Ntrain}}\Wdec^\top)$
    \COMMENT{decode once}
    \STATE $\mathcal{L}\gets
    \operatorname{CE}(\hat{\vy},\vy)$
    \RETURN $\mathcal{L},\nabla_\theta\mathcal{L}$
  \end{algorithmic}
\end{algorithm}


\subsection{What the Intermediate States Carry}
\label{sec:superposition}

\begin{figure}[!t]
  \centering
  \includegraphics[width=\fullwidth]{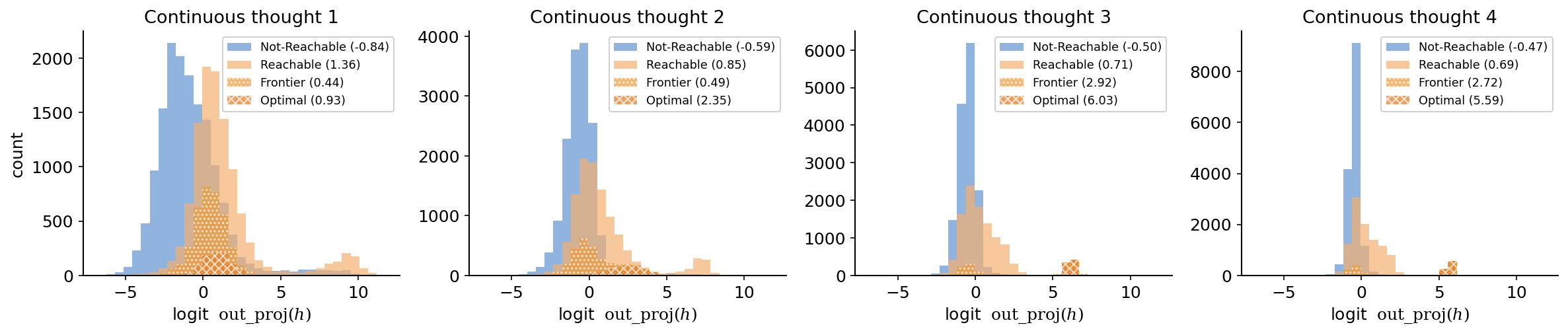}
  \caption{\textbf{Intermediate states separate reachable vertices.}
  We plot the logit $\langle\vz_t,\vw_v\rangle$ of vertex $v$, where $\vw_v$
  is row $v$ of the final decoder $\Wdec$. Logits are centered separately for
  each problem by their mean over vertices. UFM trained with $\Ntrain=5$ is evaluated at $t=0.2,0.4,0.6,0.8$, on the $217$ ProsQA test problems with $\dist(\rootnode,\ansv)=4$, averaged
  over four seeds. Vertices are grouped as \emph{not reachable},
  \emph{reachable}, \emph{frontier} ($\dist(\rootnode,v)=k$), and
  \emph{optimal} (the vertex at hop $k$ on a shortest root-to-answer path).}
  \label{fig:superposition}
\end{figure}

Following \Cref{sec:theory}, we inspect the logit
$\langle\vz_t,\vw_v\rangle$ of each vertex $v$, where $\vw_v$ is its
output embedding. In \Cref{fig:superposition}, non-reachable vertices remain near zero at every
step, while several reachable vertices receive positive logits. At the first
step, no single reachable vertex is clearly selected. The reachable group sits above the not-reachable one
and drifts further right as $t$ grows, and within it the optimal node separates the most.

While it does not establish that the trained model implements the construction, this qualitatively resembles \Cref{cor:euler}, whose Euler states have
support on all vertices reachable within $k$ hops after $k$ steps.



\section{Scaling to Long Rollouts}
\label{sec:scaling}

We evaluate UFM on Sudoku-Hard, Sudoku-Extreme, and Maze-Hard.
These tasks require longer answer sequences and longer
rollouts than ProsQA.

\subsection{Benchmarks and evaluation protocol}

\textbf{Sudoku-Hard.}
We use the split of \citet{HypersphericalFlows2026}: $48$k training
puzzles and $2$k reporting puzzles, each with $30$ givens.
UFM checkpoints are selected on an additional $2$k validation split,
disjoint from both sets. Our best FLM and S-FLM checkpoints are selected
on the reporting set.

\textbf{Sudoku-Extreme.}
Following \citet{wang2025hrm,JolicoeurMartineau2025}, training uses
$1$k base puzzles with $1000$ augmentations each. We evaluate on the full
test set of $422{,}786$ puzzles. Checkpoint selection uses $2$k additional puzzles from the
original training pool, disjoint from the selected $1$k base
training puzzles and the test set. Each answer contains $81$ positions.

\textbf{Maze-Hard.}
We use the $30\times30$ benchmark of
\citet{wang2025hrm,JolicoeurMartineau2025}, with $1000$ training and
$1000$ test mazes whose shortest paths exceed $110$ cells.
The model predicts a $900$-cell grid, of which roughly $110$ cells lie on the solution path. UFM and both flow baselines select checkpoints on the test set; no separate validation split is used.

\subsection{Training long rollouts}
\paragraph{Sphere retraction.}

On long Sudoku rollouts, the latent norm can grow sharply
(\Cref{app:sphere}). This changes the scale of the inputs to
later updates and can destabilize the rollout. We control this
growth on Sudoku and Maze with a position-wise sphere retraction,

\[
  \psi(\vz)=\sqrt{d}\,\frac{\vz}{\lVert\vz\rVert_2}.
\]
Let $\vz_{t_k}$ denote the state read by the DiT and
$y_{t_k}=\psi(\vz_{t_k})$ its retracted version. The update is 
\[
  \vz_{t_{k+1}} = y_{t_k} +\frac{\Delta t_k}{1-t_k} \bigl(\dnz(\vz_{t_k},t_k,c)-y_{t_k}\bigr), \qquad y_{t_{k+1}} = \psi(\vz_{t_{k+1}}), \text{ for } 1\le k<N .
\]
This is a projected Euler-type update: after the first step, its base is
retracted, while the DiT continues to read the raw state. We use this retraction on Sudoku and Maze only. In a separate Sudoku
ablation, training without it lowers accuracy; details are in
\Cref{app:sphere}.

\paragraph{Truncated Backpropagation}

For the long tasks, we reduce activation memory by backpropagating only
through the final $\Nback=6$ of $\Ntrain=24$ updates, as in
\citet{JolicoeurMartineau2025,clark2024draft}. The first $18$ updates
are computed without recording gradients and produce the state refined
by the final six. We choose $\Nback=6$ to fit the available memory.

Each training example uses a newly sampled increasing time grid ending at
$t=1$, exposing the model to varying step sizes. The complete
long-rollout procedures are given in \Cref{app:long-rollout}.

\subsection{Comparison with Flow Baselines}
\label{sec:flow-comparison}

We compare all methods at $\Neval=128$ on Sudoku-Hard
\citep{HypersphericalFlows2026}. FLM and S-FLM use the authors'
released code and the best tested checkpoints and final-step settings
at this budget (\Cref{app:sflm-replication}).

\begin{table}[t]
\centering
\caption{Accuracy (\%) on Sudoku-Hard. FLM and S-FLM are our runs
of the authors' released code. Train is training-only wall-clock time on one
NVIDIA A100. All methods use the same training data.}
\label{tab:sflm_hard}
\begin{tabular}{lccc}
\toprule
Method & Params & Train (h) & Sudoku-Hard \\
\midrule
FLM \citep{FlowMapLM2026}$^{a}$                        & $28.6$M & $1.6$ & 51.9 \\
S-FLM \citep{HypersphericalFlows2026}$^{a}$ & $28.6$M & $1.7$ & 50.9$^{b}$ \\
\midrule
UFM (ours) & $8.4$M & $1.1$ & $\mathbf{86.9}\pm0.9$ \\
\bottomrule
\end{tabular}
\vspace{0.4em}
\parbox{\linewidth}{
$^{a}$Selected on the $2$k reporting puzzles; UFM is selected on a separate
$2$k validation split. The UFM result is the mean and standard deviation over five independent
noise seeds at $\Neval=128$ on a uniform grid.
$^{b}$With ancestral noise removal at $\Neval=128$; with the default greedy
final step, S-FLM peaks at $50.2\%$ at $N=64$
(\Cref{tab:acc-vs-n}).}
\end{table}

On Sudoku-Hard, UFM reaches $86.9\pm0.9\%$, compared with $51.9\%$
for FLM and $50.9\%$ for S-FLM, using $8.4$M rather than
$28.6$M parameters. 

The gap widens on the harder tasks. On Sudoku-Extreme, UFM reaches
$74.4\pm0.0\%$ with a uniform $\Neval=128$ grid, averaged over five noise
seeds. FLM and S-FLM reach $10.7\%$ and $9.4\%$, respectively. On
Maze-Hard, UFM reaches $89.3\pm0.5\%$, compared with $40.4\%$ for FLM and
$49.9\%$ for S-FLM (\Cref{tab:main-single}). Published recursive-reasoner results and FRM are included for context in \Cref{tab:main-single,tab:main-multi}.

\subsection{Selecting Among Rollouts}
\label{sec:chooser}

Inference begins from Gaussian noise, so different rollouts can produce
different answers. We therefore draw $K=100$ rollouts and select one with a
parameter-free score. \emph{Margin} averages the gap between the two largest raw logits over
output positions. It is an output-space analogue of $\mgn$ in
\Cref{thm:existence}.

On mazes, scores average over $900$ cells, although only roughly $110$
describe the solution path. This makes selection harder, and selected
accuracy remains below Pass@100. Detailed
diagnostics are in \Cref{app:chooser}.

\section{Related work}

\paragraph{Recursive and looped reasoning.}
Recursive reasoners repeatedly apply a small, weight-tied network to a latent state. HRM \citep{wang2025hrm} uses slow and fast recurrent modules, while TRM \citep{JolicoeurMartineau2025} uses a single recurrent block and an answer state. Subsequent work improves stability and scaling \citep{fprm,attractor,urm}.
EqR and GRAM explore both depth, through more iterations, and breadth,
through multiple trajectories \citep{eqr,gram}. PTRM also studies sampling
and selection \citep{ptrm}.

\paragraph{Latent chain of thought.}
Coconut \citep{Hao2024Coconut} feeds a language model's hidden state back as its next input embedding and introduces ProsQA. \citet{zhu2025reasoning} show that a two-layer Transformer with continuous thoughts can solve graph reachability by maintaining a superposition of reachable vertices and expanding it over several steps. Their result provides the starting point for our construction in \Cref{sec:theory}. UFM updates one latent answer state in place and uses answer-only
supervision throughout training. Coconut appends continuous thoughts
and uses a curriculum that gradually replaces language reasoning steps.

\paragraph{Differentiating through generation.}
Backpropagating a terminal objective through a generative sampler is established in diffusion reward fine-tuning. DRaFT \citep{clark2024draft} backpropagates differentiable rewards through a diffusion sampler and introduces DRaFT-\(K\), which truncates gradients to the final \(K\) sampling steps. 
UFM uses the same general optimization mechanism, but is trained from scratch with endpoint cross-entropy on reasoning tasks rather than fine-tuned against an external reward. Our \(\Nback\) plays the role of a truncated sampler-backpropagation horizon.

\paragraph{Flow models for reasoning.}
Flow Reasoning Models (FRM) \citep{helbling2026frm} add recurrent
self-conditioning to a flow model. Their Fixed-Point Forcing procedure uses a
carry produced by the model's own rollout, while stopping gradients through that
carry. They also report that increasing the number of sampling steps of a
standard FLM does not improve Sudoku performance, consistent with our ProsQA
observation in \Cref{sec:failure}. FRM reports a peak accuracy of $99.5\%$ on a fixed
$1000$-puzzle subset of the Sudoku-Extreme test set,
included for context in \Cref{tab:main-single}.
Its maze evaluation uses Maze-Unique rather than Maze-Hard.

Looped Flows \citep{suleymanzade2026thinkingloopedflows}, concurrent with this
work, follows the TRM preprocessing and evaluation protocol on Sudoku-Extreme
and Maze-Hard. It trains recurrent flow updates with local denoising losses and
stops gradients between updates. UFM instead applies one terminal loss and
backpropagates through the final $\Nback$ updates of the rollout.

\section{Limitations}
Our experiments use relatively small models, from $8$M to $16$M parameters, on
synthetic structured reasoning tasks. Training through a rollout requires memory
for the backpropagated suffix, which limited model size in this work. Extending
UFM to larger models and less structured language-reasoning tasks will require
more memory-efficient ways to train consecutive updates jointly.

The construction in \Cref{sec:theory} is an existence result. It relies on
orthonormal embeddings and a linear propagation head, and does not show that
gradient descent recovers this velocity field. The trained models share only a
qualitative pattern with the construction, examined in
\Cref{sec:superposition}.

Finally, a rollout begins from Gaussian noise, so a correct answer may appear
among several trajectories without being selected by a single rollout. Our
parameter-free selection rule recovers most of the Pass@100 gap on
Sudoku-Extreme, but is less reliable for the much longer Maze-Hard outputs and could benefit from a learned correctness or halting head, as used
in HRM and TRM \citep{wang2025hrm,JolicoeurMartineau2025}.
Although UFM substantially improves over standard flow baselines, it remains
behind the strongest reported recursive reasoners on Sudoku-Extreme.

\section{Conclusion}

We show how flow models can use additional integration steps
to improve reasoning. On graph reachability, a
two-layer construction propagates information along the graph, and its Euler
discretization requires at least as many steps as the distance from the root to
the answer. Standard flow language models do not use their steps in this way:
their accuracy remains largely flat as the number of sampling steps increases. UFM changes this by training through the model's own latent rollout with a
single loss at its endpoint. Accuracy then grows with the number of steps on
ProsQA. For longer Sudoku and Maze rollouts, sphere retraction stabilizes the
updates and allows an $8.4$M-parameter model to
outperform $28.6$M standard-flow baselines. Multiple rollouts further improve
reliability, although selecting among them is less effective for long answers.

Applying this approach to larger language models will require more
memory-efficient rollout training and stronger selection mechanisms. Our results show that training through the rollout can make additional flow steps useful for reasoning.

\section*{Acknowledgements}
This work was supported in part by the Hector Foundation, by the Max Planck ETH Center for Learning Systems, and by EPSRC Turing AI World-Leading Research Fellowship No. EP/X040062/1 and EPSRC AI Hub on Mathematical Foundations of Intelligence: An "Erlangen Programme" for AI No. EP/Y028872/1.






\bibliography{iclr2027_conference}
\bibliographystyle{iclr2027_conference}

\clearpage

\appendix
\section*{Appendix Overview}

The appendix provides theoretical proofs, training and sampling
algorithms, and additional experimental details and results.

\begin{itemize}
    \item Appendix~\ref{app:proofs} gives the graph-reachability
    construction and proofs, including Euler discretization,
    noisy initialization, and full-vocabulary readout.

    \item Appendix~\ref{app:sphere} describes sphere retraction
    and reports the retraction ablation and norm diagnostic.

    \item Appendix~\ref{app:long-rollout} provides the training
    and inference algorithms for long rollouts.

    \item Appendix~\ref{app:exp} gives the experimental settings,
    evaluation procedures, and additional results.
\end{itemize}

\section{Graph Reachability: Construction and Proofs}
\label{app:proofs}

We give the construction and proofs supporting \Cref{sec:theory}.
We describe the prompt representation and attention layers, then
establish the exact-flow solution, its Euler discretization, and
the noisy-initialization bound.

\subsection{Notation and assumptions}
\label{app:notation}

Notation is collected in \Cref{tab:notation}. Fix a maximum graph
size $\nmax\ge2$ and a vocabulary containing the vertex labels and
the special tokens. These special tokens are distinct from all
vertex labels. Graphs are directed, may contain cycles, and have
distinct edges, so $\nedge\le\nmax^2$.

Each token has an embedding $\emb{v}\in\mathbb{R}^{\dte}$.
As in \citet{zhu2025reasoning}, the construction assumes
$\embmat^\top\embmat=\mI$, which requires $\dte\ge V$.
These assumptions apply to the construction; the trained models
use learned embeddings.

We set $\dist(\rootnode,v)=\infty$ for unreachable vertices.
The walk count includes zero-length walks:
$\walk{0}(u,v)=1$ when $u=v$ and $0$ otherwise.
The evolving state $\vz_t$ is the content slot of $\tokA$;
the prompt information remains in the cache.

\begin{table}[h]
\centering\small
\caption{Notation for \Cref{sec:theory} and this appendix. The last column gives
the symbol of \citet{zhu2025reasoning} where one exists; we rename their $t_j$
and $\tilde u_v$ to free $t$ for flow time and $u,v$ for vertices.}
\label{tab:notation}
\begin{tabular}{@{}lp{6.2cm}l@{}}
\toprule
Symbol & Meaning & Zhu et al. \\
\midrule
$\gph=(\vertices,\edges)$ & directed graph, $\nv=|\vertices|\le\nmax$, $\nedge=|\edges|\ge 0$ distinct edges & same \\
$(\eg{j}\!\to\!\eh{j})$ & source and target of edge $j$ & $(s_j,t_j)$ \\
$\rootnode$; $\cand{1},\cand{2}$; $\ansv$, $\wrongv$ & root; candidates; reachable, unreachable candidate & same \\
$\dist(\cdot,\cdot)$; $\diam=\ecc(\rootnode)$ & graph distance; eccentricity of the root & --- \\
$\reach{k}$ & vertices within $k$ hops of $\rootnode$ & $\mathcal V_c$ \\
$\walk{k}(u,v)$ & number of directed walks of length $k$ from $u$ to $v$ & --- \\
\midrule
$\seqlen=3\nedge+7$; $\Tmax=3\nmax^2+7$ & prompt length; its bound & --- \\

$\tokBOS,\tokEdge,\tokQ,\tokR,\tokA$ & special tokens & same \\
$\emb{v}\in\mathbb R^{\dte}$, $\embmat$ & token embedding, $\embmat^\top\embmat=\mI$ & $\tilde u_v$ \\
$\ct,\bufa,\bufb,\bufc,\bufd,\pe$ & residual-stream slots & content, buffers, pos \\
$\pos{a}$; $\epe$ & sinusoidal encoding of position $a$; separation gap & same; --- \\
$\pidx(\cdot)$ & position in the prompt & same \\
$\chsr{x}{\ell}$; $\kappa,\beta$ & attention chooser (token $x$, offset $\ell$); its gains & --- \\
$\epsth$ & threshold & --- \\
\midrule
$\vz_t\in\mathbb R^{\dte}$ & content slot of $\tokA$ at flow time $t$ & --- \\
$\Prop$; $\gain$ & propagation operator; value gain & --- \\
$\Gamma_{u,v}$; $\mgn$ & $\sum_k \gain^k\walk{k}(u,v)/k!$; readout margin & --- \\
$\nsteps$ & number of Euler steps & $C$ (thoughts) \\
$\nse\sim\mathcal N(0,\mI)$; $\sigma$ & initialization noise; its scale & --- \\
$M_0$ & readout gain on the candidates & --- \\

\bottomrule
\end{tabular}
\end{table}

\subsection{Prompt format}
\label{app:prompt-figure}

\begin{figure}[t]
  \centering
  \resizebox{\textwidth}{!}{%
    \begin{tikzpicture}[
        node distance=2.5mm,
        box/.style={
          rectangle,
          rounded corners=3pt,
          inner sep=4pt,
          font=\sffamily,
          minimum height=6mm,
          minimum width=6mm,
          align=center
        },
        pinkbox/.style={box, fill=boxpink},
        bluebox/.style={box, fill=boxblue},
        orangebox/.style={box, fill=boxorange},
        greenbox/.style={box, fill=boxgreen},
        arr/.style={
          ->,
          >=Stealth,
          line width=1.5pt,
          color=arrowblue
        },
        labeltext/.style={
          font=\sffamily\footnotesize\color{textgray},
          align=center
        }
      ]
 
      \node[pinkbox]              (s)     {\textless s\textgreater};
      \node[bluebox,   right=of s]    (s1)    {$u_1$};
      \node[bluebox,   right=of s1]   (t1)    {$v_1$};
      \node[pinkbox,   right=of t1]   (e1)    {\textless e\textgreater};
      \node[bluebox,   right=of e1]   (s2)    {$u_2$};
      \node[bluebox,   right=of s2]   (t2)    {$v_2$};
      \node[pinkbox,   right=of t2]   (e2)    {\textless e\textgreater};
      \node[pinkbox,   right=of e2]   (dots)  {\dots};
      \node[bluebox,   right=of dots] (sm)    {$u_m$};
      \node[bluebox,   right=of sm]   (tm)    {$v_m$};
      \node[pinkbox,   right=of tm]   (em)    {\textless e\textgreater};
      \node[pinkbox,   right=of em]   (Q)     {\textless Q\textgreater};
      \node[orangebox, right=of Q]    (c1)    {$c_1$};
      \node[orangebox, right=of c1]   (c2)    {$c_2$};
      \node[pinkbox,   right=of c2]   (R)     {\textless R\textgreater};
      \node[greenbox,  right=of R]    (r)     {$r$};
 
      \draw[arr] (s.south) -- ++(0,-0.5) node[below, labeltext] {start token};
 
      \draw[color=arrowblue, line width=1.2pt, decorate,
            decoration={brace, amplitude=4pt, raise=2pt}]
        (s1.north west) -- (t1.north east)
        node[midway, above=8pt, labeltext] {directed edge};
 
      \draw[arr] (e1.north) -- ++(0,0.8) node[above, labeltext] {special edge token};
 
      \draw[arr] (s2.south) -- ++(-0.25,-0.5)
        node[below, labeltext, xshift=-0.2cm] {source node};
      \draw[arr] (t2.south) -- ++(0.25,-0.5)
        node[below, labeltext, xshift=0.2cm] {target node};
 
      \draw[arr] (Q.north) -- ++(0,0.5) node[above, labeltext] {question token};
 
      \coordinate (c_mid) at ($(c1.south)!0.5!(c2.south)$);
      \node[labeltext, below=0.6cm of c_mid] (cand_label) {candidate nodes};
      \draw[arr] (c1.south) -- (c1.south |- cand_label.north);
      \draw[arr] (c2.south) -- (c2.south |- cand_label.north);
 
      \draw[arr] (R.north) -- ++(0,0.5) node[above, labeltext] {reasoning token};
 
      \draw[arr] (r.south) -- ++(0,-0.5) node[below, labeltext] {root node};
    \node[greenbox, right=of r] (A) {\textless A\textgreater};
   \draw[arr] (A.north) -- ++(0,0.8) node[above, labeltext] {answer token};
   
    \end{tikzpicture}%
  }
  \caption{Visualization of the prompt token sequence.}
  \label{fig:prompt}
\end{figure}
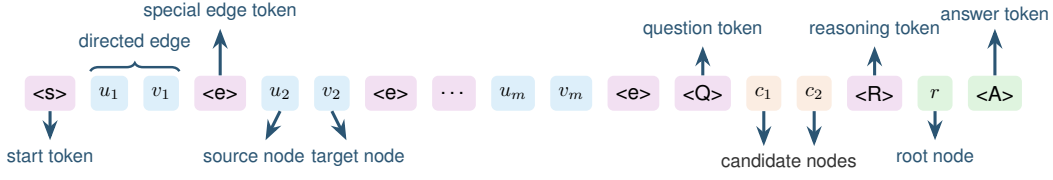

\subsection{Overview}
\label{app:overview}

\paragraph{Relation to \citet{zhu2025reasoning}.}
The first layer follows their construction: five attention choosers
(\Cref{lem:chooser}) cache the information needed later, and a threshold MLP
removes attention leakage. We adapt their construction as follows. (i) As in their paper, we state the construction with causal attention; \Cref{rem:bidirectional} shows that the same weights work with the
bidirectional attention of our models. (ii) The last three
heads read from $\tokA$: two cache the candidates there, and the fifth writes the
root $\rootnode$ into the content slot of $\tokA$, which sets the initial state of
the flow. (iii) Their second layer, which expands the frontier by one hop per
continuous thought with softmax attention and a threshold MLP, is replaced by a
single linear attention head whose output defines the velocity field.

\paragraph{What is evaluated along the flow.}
Layer~1 is applied once to the prompt: it fills the buffers and sets the initial
state $\vz_0$ at $\tokA$. At every time $t$, the velocity is the output of the
propagation head with the current state $\vz_t$ as query. Neither the chooser heads nor the threshold MLP act on $\vz_t$, and the
construction uses no normalization layers, so the velocity is linear in the state.

\subsection{Transformer preliminaries}
\label{app:prelim}

\paragraph{Prompt length.}
With $\tokA$ appended, the prompt has length $\seqlen = 3\nedge + 7$. Since the
edges are distinct, $\nedge \le \nmax^2$, so $\seqlen \le \Tmax := 3\nmax^2 + 7$.
All gains below are chosen from $\Tmax$, so they do not depend on the graph.

\paragraph{Residual stream.}
We split the residual stream of width $d = 5\dte + \dpe$ into six slots: a content
slot $\ct$ and four buffers $\bufa, \bufb, \bufc, \bufd$, all of width $\dte$, and
a positional slot $\pe$ of width $\dpe$. Token embeddings live in the content slot
and are orthonormal (\Cref{app:notation}): $\embmat^\top\embmat = \mI$, where
$\embmat = [\emb{1}, \dots, \emb{V}]$ and $V = |\vocab|$.

In full-stream expressions, token embeddings are zero-padded outside
the content slot. We keep positional encodings separate as $\vp_i$;
the displayed cached states omit this positional block.

\paragraph{Attention.}

One head of parameters $(\mW_Q, \mW_K, \mW_V, \mW_O)$ computes the query $q_i = \mW_Q(\vh_i + \vp_i)$, the keys $k_j = \mW_K(\vh_j + \vp_j)$, and the attention scores \[
  s_{i,j} = \frac{\exp\langle q_i, k_j\rangle}
                 {\sum_{j'} \exp\langle q_i, k_{j'}\rangle},
\] where $j'$ ranges over $j' \le i$ with causal masking and over all of $[\seqlen]$
in the bidirectional case (\Cref{rem:bidirectional}).

\begin{definition}[Sinusoidal Positional Encoding]
    \label{def:pe}
    
    Fix an even positional dimension $\dpe\ge2$, set $M=\Tmax+4$, and
    let $\omega=M^{-2/\dpe}$. The positional encoding
    $\pos{i}\in\mathbb{R}^{\dpe}$ has entries:

    \[
  \bar p_{i,2k-1} = \cos(i\,\omega^k), \qquad
  \bar p_{i,2k}   = \sin(i\,\omega^k), \qquad k \in [\dpe/2],
\]
and $\vp_i$ is the vector living in the positional slot with $\pe(\vp_i) = \pos{i}$.
    
\end{definition}

\begin{lemma}[Separation, \citealt{zhu2025reasoning}]
\label{lem:pe-separation}

With the choices in \Cref{def:pe}, $\langle \pos{i}, \pos{i}\rangle = \dpe/2$
for all $i$, and there exists $\epe > 0$, depending only on $\Tmax$, such that
$\langle \pos{i}, \pos{j}\rangle \le \dpe/2 - \epe$ for all $i \neq j$ in $[\Tmax+4]$.
Hence, for $\ell \le 4$ and $\ell < i \le \seqlen$, $\langle \pos{i}, \pos{j+\ell}\rangle$
is maximized over $j \in [\seqlen]$ at the unique index $j = i-\ell$, with a gap of at
least $\epe$.

\end{lemma}

\begin{lemma}[Attention chooser, \citealt{zhu2025reasoning}]
\label{lem:chooser}
Consider a token $x \in \vocab$, an offset $\ell \in \{1,2,3,4\}$ and $\varepsilon \in (0,1)$.
Under \Cref{def:pe}, there exist $\mW_Q, \mW_K$ for a causal head, with gains depending only on $\Tmax$ and $\varepsilon$, such that for any $\seqlen \le \Tmax$ and any input sequence satisfying:

\begin{enumerate}
\item[(H1)] $\vh_i = \sum_{v \in \vocab} \alpha_{i,v}\,\emb{v}$ with $\alpha_{i,v} \ge 0$ and $\sum_v \alpha_{i,v}^2 = 1$, for all $i$;
\item[(H2)] $\langle \emb{x}, \vh_i\rangle \in \{0,1\}$ for all $i$, so each state is either equal to $\emb{x}$ or orthogonal to it;
\item[(H3)] $\langle \emb{x}, \vh_i\rangle = 0$ for $i \le \ell$, so none of the first $\ell$ positions holds $x$.

\end{enumerate}

We have, for all $i \in [\seqlen]$,
\[
  s_{i,\,i-\ell} > 1-\varepsilon \ \text{ if } \langle \emb{x},\vh_i\rangle = 1,
  \qquad
  s_{i,1} > 1-\varepsilon \ \text{ otherwise.}
\]

We write $\chsr{x}{\ell}$ for such a head.
\end{lemma}

The idea is the following. Since the encoding is sinusoidal,
$\pos{j+\ell} = \bm{R}_\ell\,\pos{j}$ for a fixed rotation $\bm{R}_\ell$, so a
linear key can contain both $\pos{j+\ell}$ and $\pos{j}$. Taking
$q_i \propto \big(\pos{i},\ \kappa\langle\emb{\bar x},\vh_i\rangle\,\pos{1}\big)$
and $k_j \propto \big(\bm{R}_\ell\pos{j},\ \pos{j}\big)$ in two orthogonal
subspaces, with $\emb{\bar x} = \sum_{v\neq x}\emb{v}$, gives
\[
  \langle q_i, k_j\rangle
  = \beta\Big(\langle \pos{i}, \pos{j+\ell}\rangle
    + \kappa\,\langle \emb{\bar x}, \vh_i\rangle\,\langle \pos{1}, \pos{j}\rangle\Big).
\]
If position $i$ holds $x$, then $\langle\emb{\bar x},\vh_i\rangle = 0$ and
\Cref{lem:pe-separation} puts the unique maximum at $j = i-\ell$ (and $i > \ell$ by (H3), so $i-\ell \ge 1$). Otherwise
$\langle\emb{\bar x},\vh_i\rangle \ge 1$ by (H1)--(H2), and for $\kappa$ large the
second term, uniquely maximized at $j = 1$, dominates. A large $\beta$ turns both
gaps into attention weights above $1-\varepsilon$. We refer to
\citet{zhu2025reasoning} for the complete proof.

Our five heads use $x\in\{\tokEdge,\tokA\}$ and offsets
$\ell\in\{1,2,3,4\}$. Each input position carries one token
embedding, so (H1)--(H2) hold. Edge tokens occur at positions
$3j+1\ge4$, while the answer token occurs at position
$3\nedge+7\ge7$. Since the edge heads use offsets at most two
and the answer heads use offsets at most four, (H3) also holds.

\begin{remark}[Bidirectional attention]
\label{rem:bidirectional}
In a bidirectional head, the softmax runs over all $j \in [\seqlen]$ instead of
$j \le i$, and the score is unchanged:
\[
  \langle q_i, k_j\rangle
  = \beta\Big(\langle \pos{i}, \pos{j+\ell}\rangle
    + \kappa\,\langle \emb{\bar x}, \vh_i\rangle\,\langle \pos{1}, \pos{j}\rangle\Big).
\]
For a head that does not fire, $\langle \emb{\bar x}, \vh_i\rangle \ge 1$ by
(H1)--(H2), and by \Cref{lem:pe-separation}, $\langle \pos{1}, \pos{j}\rangle$ is
uniquely maximized at $j = 1$ over all of $[\seqlen]$, not only over $j \le i$; for
$\kappa$ large, the head attends to $\tokBOS$. For a head that fires,
$\langle \emb{\bar x}, \vh_i\rangle = 0$ and the score is
$\beta\langle \pos{i}, \pos{j+\ell}\rangle$; since
$j+\ell \le \seqlen + 4 \le \Tmax + 4$ for every $j \in [\seqlen]$,
\Cref{lem:pe-separation} gives the unique maximum at $j = i-\ell$ with the same
gap $\epe$. The softmax runs over at most $\Tmax$ positions in both cases, so the
gains $\beta, \kappa$ of \Cref{lem:chooser} are unchanged, and \Cref{lem:cache}
holds. The propagation head is unaffected, since $\tokA$ is the last position.
\end{remark}


\subsection{Layer 1: caching}
\label{app:static-memory}

\begin{lemma}[Static memory and answer initialization]
\label{lem:cache}
Append $\tokA$ immediately after the root. One pass of the five chooser heads
\[
\chsr{\tokEdge}{2},\quad \chsr{\tokEdge}{1},\quad
\chsr{\tokA}{4},\quad \chsr{\tokA}{3},\quad \chsr{\tokA}{1},
\]
followed by a threshold MLP, stores the source and target of each edge at its
edge token, stores $\emb{\cand1}$ and $\emb{\cand2}$ in separate buffers at
$\tokA$, and copies $\emb{\rootnode}$ into the content slot of $\tokA$. The
relevant states are
\begin{equation}
 \vh_{\tokEdge,j}=
 \begin{bmatrix} \emb{\tokEdge}\\ \emb{\eg{j}}\\ \emb{\eh{j}}\\ \mZero\\ \mZero \end{bmatrix},
 \qquad
 \vh_{\tokA}=
 \begin{bmatrix} \emb{\tokA}+\emb{\rootnode}\\ \mZero\\ \mZero\\ \emb{\cand1}\\ \emb{\cand2} \end{bmatrix},
 \qquad\text{so that}\qquad
 \vz_0=\emb{\tokA}+\emb{\rootnode}.
\end{equation}
The key and value of the propagation head (\Cref{def:prop}) read $\bufa$ and
$\bufb$ of $\vh_{\tokEdge,j}$, and $\bufa$ is zero at every non-edge position.

Non-firing heads assign more than $1-\varepsilon$ attention to
$\tokBOS$, whose value is zero. We choose
$\epsth=1/(2\sqrt{\nmax})$ and $\varepsilon=\epsth/2$.
The threshold and chooser gains can therefore be fixed using
only $\nmax$.

\end{lemma}

\begin{proof}
Each of the five heads is an attention chooser, so each one puts at least
$1-\varepsilon$ of its attention on the position it is aimed at. We check that the
offsets land correctly. With $\tokA$ appended after the root, the positions are
$\pidx(\tokBOS)=1$, $\pidx(\eg{j})=3j-1$, $\pidx(\eh{j})=3j$,
$\pidx(\tokEdge,j)=3j+1$, $\pidx(\tokQ)=3\nedge+2$,
$\pidx(\cand{1})=3\nedge+3$, $\pidx(\cand{2})=3\nedge+4$,
$\pidx(\tokR)=3\nedge+5$, $\pidx(\rootnode)=3\nedge+6$ and $\pidx(\tokA)=3\nedge+7$.
So $\chsr{\tokEdge}{2}$ and $\chsr{\tokEdge}{1}$ read $\eg{j}$ and $\eh{j}$; $\chsr{\tokA}{4}$ and $\chsr{\tokA}{3}$ read $\cand{1}$ and $\cand{2}$; and
$\chsr{\tokA}{1}$ reads $\rootnode$.

As in \citet{zhu2025reasoning}, all heads use the value matrix, $\mW_V=[\,\mI_{\dte}-\emb{\tokBOS}\emb{\tokBOS}^{\top}\ \ \mZero\,]$, which will send $\emb{\tokBOS}$ to zero. 

Heads~1--4 write into buffers~1--4, respectively, and head~5
writes into the content slot. A non-firing head assigns more
than $1-\varepsilon$ attention to $\tokBOS$, whose value is zero.
Each remaining output coefficient is therefore smaller than
$\varepsilon$ and is removed by the threshold MLP.

As in \citet{zhu2025reasoning}, a threshold MLP with threshold $\epsth$, applied
to each coefficient of the content slot and the buffers in the embedding basis,
maps coefficients at most $\epsth$ to $0$ and coefficients at least $1-\epsth$
to $1$. 

Each head puts attention above $1-\varepsilon$ on its target and leaks less than
$\varepsilon$ elsewhere, so with $\varepsilon < \epsth$ every coefficient becomes
exactly $0$ or $1$, and $\epsth$ does not depend on the graph.

Only head~1 writes into buffer~1, and it fires only at edge tokens, so after
thresholding buffer~1 is zero at every other position. Since buffer~1 is the key
of the propagation head, only edge tokens contribute to its output.

Buffer~2 is written only by head~2, which fires at edge tokens, and buffers~3
and~4 only by heads~3 and~4, which fire at $\tokA$. Hence at $\pidx(\tokA)$
buffers~1 and~2 are empty, buffers~3 and~4 hold $\emb{\cand1}$ and $\emb{\cand2}$,
and the content slot holds $\emb{\tokA}$ together with the root written by
head~5, so $\vz_0 = \emb{\tokA} + \emb{\rootnode}$, as claimed.

\end{proof}

\subsection{Layer 2: linear propagation}
\label{app:linear-attention}

\begin{definition}[Linear propagation head]
\label{def:prop}
The second layer is a single unnormalized linear-attention head \citep{katharopoulos2020transformers}, with 
\[
  \mW_Q = [\,\mI_{\dte}\ \ \mZero\,], \quad
  \mW_K = [\,\mZero_{\dte}\ \ \mI_{\dte}\ \ \mZero\,], \quad
  \mW_V = \gain\,[\,\mZero_{\dte}\ \ \mZero_{\dte}\ \ \mI_{\dte}\ \ \mZero\,],
  \quad
  \mW_O = [\,\mI_{\dte}\ \ \mZero\,]^{\top},
\]

in a way that it reads the content slot as a query, the buffer $1$ as key and the buffer $2$ as a value with a gain $\gain$, and writes back on the content slot. There is no softmax and no normalizing denominator: 
\[
  o_i = \mW_O \sum_{j \le i}
        \langle \mW_Q \vh_i,\, \mW_K \vh_j\rangle\, \mW_V \vh_j .
\]
(the sum runs over all $j$ in the bidirectional case; at $\tokA$, the last
position, the two coincide)

\end{definition}

Define
\begin{equation}
 \Prop:=\sum_{(\eg j\to\eh j)\in\edges}
 \emb{\eh j}\,\emb{\eg j}^{\top}.
\end{equation}

By \Cref{lem:cache}, the buffer $1$ is non-zero only at  the edge tokens, where it contains $\emb{\eg j}$ and the buffer $2$ contains $\emb{\eh j}$. Hence, with a query $\vz$ at $\tokA$, the head returns $\gain\Prop\,\vz$. Moreover, the terms of the sum $\emb{\eh j}\emb{\eg j}^{\top}$ are orthonormal in the Frobenius sense for distinct edges, thus \[
  \lVert\Prop\rVert_2 \le \lVert\Prop\rVert_F = \sqrt{\nedge}.
\]
\subsection{Proof of \texorpdfstring{\Cref{thm:existence}}{Theorem 1}}
\label{app:proof-thm}

Because $\emb{\tokA}$ is orthogonal to every vertex embedding,
\begin{equation}
 \Prop\,\emb{\tokA}=\mZero,
\end{equation}
so under
\begin{equation}
 \dot{\vz}_t=\gain \Prop \vz_t,
 \qquad \vz_0=\emb{\tokA}+\emb{\rootnode},
 \qquad t\in[0,1],
\end{equation}
the answer-token identity remains fixed.

Writing
\begin{equation}
 \vz_t=\emb{\tokA}+a_t
\end{equation}
is only a mathematical decomposition of the single residual state, not a separate architectural block. Its vertex component satisfies
\begin{equation}
 \dot a_t=\gain \Prop a_t,
 \qquad a_0=\emb{\rootnode}.
\end{equation}

\begin{lemma}[Reachability from the matrix exponential]
\label{lem:walks}
Let $\walk{k}(\rootnode,v)$ be the number of directed walks of length $k$ from $\rootnode$ to $v$. Then
$a_t=e^{\gain t\Prop}\emb{\rootnode}$ and, for every $v\in\vertices$,
\begin{enumerate}
\item[(i)] $\displaystyle \langle\emb{v},\vz_t\rangle=\langle\emb{v},a_t\rangle=\sum_{k\ge0}\frac{(\gain t)^k}{k!}\walk{k}(\rootnode,v)\ge0$;
\item[(ii)] for $t > 0$, this coordinate is positive iff $v$ is reachable from $\rootnode$, and is exactly $0$ otherwise;
\item[(iii)] if $v$ is reachable, then $\displaystyle \langle\emb{v},\vz_1\rangle\ge\mgn:=\min_{0\le k\le \nmax-1}\frac{\gain^k}{k!}>0$.
\end{enumerate}
\end{lemma}

\begin{proof}
Orthonormality gives $\Prop\,\emb{x}=\sum_{j:\eg j=x}\emb{\eh j}$, hence by induction $\Prop^k\emb{u}=\sum_v\walk{k}(u,v)\emb{v}$ for every vertex $u$. Expanding the exponential proves (i). All terms are nonnegative, proving (ii). For reachable $v$, a shortest path has length $k^*=\dist(\rootnode,v)\le \nmax-1$ and $\walk{k^*}(\rootnode,v)\ge1$, so its single series term proves (iii). The equality $\langle\emb{v},\vz_t\rangle=\langle\emb{v},a_t\rangle$ follows from $\emb{v}\perp\emb{\tokA}$.
\end{proof}

\begin{proof}[Proof of \Cref{thm:existence}]
By \Cref{lem:cache}, layer~1 sets $\vz_0 = \emb{\tokA} + \emb{\rootnode}$. By
\Cref{def:prop}, the velocity at $\tokA$ is $\gain\Prop\vz_t$, whose time-$1$ flow
is $\vz_1 = \emb{\tokA} + e^{\gain\Prop}\emb{\rootnode}$ since
$\Prop\emb{\tokA} = \mZero$. \Cref{lem:walks}(i)--(ii) give the coefficients and
their signs; the reachable candidate has coefficient at least $\mgn$ by (iii) and
the unreachable one exactly $0$ by (ii), which gives the $\arg\max$ and the
margin. The weights depend only on $\nmax$, the fixed vocabulary, and
$\gain$, as required.
\end{proof}

\subsection{Euler discretization}
\label{app:euler}

\begin{proof}[Proof of \Cref{cor:euler}]
Since $\Prop\emb{\tokA}=\mZero$, Euler integration gives
\[
  \vz_{t_N}
  =
  \emb{\tokA}
  +
  \prod_{i=0}^{N-1}
  (\mI+\gain\Delta t_i\Prop)\emb{\rootnode},
  \qquad
  \Delta t_i=t_{i+1}-t_i>0.
\]
Expanding the product gives
$\sum_{j=0}^{N}a_j\Prop^j$, where $a_0=1$ and every $a_j>0$:
each coefficient is a sum of products of positive step sizes
and powers of $\gain$. By \Cref{lem:walks},
\[
  \langle\emb{v},\vz_{t_N}\rangle
  =
  \sum_{j=0}^{N}a_j\walk{j}(\rootnode,v).
\]
This coefficient is positive exactly when a walk of length at
most $N$ connects $\rootnode$ to $v$, equivalently when
$\dist(\rootnode,v)\le N$. The reachable candidate therefore
has a strictly larger score once
$N\ge\dist(\rootnode,\ansv)$.

If $N<\dist(\rootnode,\ansv)$, both candidate scores are zero,
so the readout does not yet separate them.

\end{proof}

\subsection{Noisy initialization}
\label{app:gaussian}

\begin{proof}[Proof of \Cref{rem:gaussian-main}]
For the exact flow with
$\vz_0=\emb{\tokA}+\emb{\rootnode}+\sigma\nse$,
where $\sigma>0$ and $\nse\sim\mathcal N(0,\mI)$,
\[
 \vz_t=\emb{\tokA}+e^{\gain t\Prop}\emb{\rootnode}+\sigma e^{\gain t\Prop}\nse.
\]
Let $g_v=\langle\emb{v},\nse\rangle$, which are i.i.d.\ $\mathcal N(0,1)$ by
orthonormality, and $\Gamma_{v,w}=\sum_k\gain^k\walk{k}(v,w)/k!\ge0$. By
\Cref{lem:walks}, for every vertex $w$,
\[
 \langle\emb{w},\vz_1\rangle=\Gamma_{\rootnode,w}+\sigma N_w,
 \qquad
 N_w=\sum_v g_v\Gamma_{v,w}\sim\mathcal N(0,\rho_w^2),
\]
with $1\le\rho_w\le e^{\gain\sqrt{\nedge}}$. Since every $\walk{k}(v,w)\ge0$, the
covariance of $N_{\cand1}$ and $N_{\cand2}$ is
$\sum_v\Gamma_{v,\cand1}\Gamma_{v,\cand2}\ge0$, so
$\mathrm{Var}(N_{\wrongv}-N_{\ansv})\le\rho_{\cand1}^2+\rho_{\cand2}^2
\le2e^{2\gain\sqrt{\nedge}}$. The readout fails only if
$\sigma(N_{\wrongv}-N_{\ansv})\ge\Gamma_{\rootnode,\ansv}$, and
$\Gamma_{\rootnode,\ansv}\ge\mgn$ by \Cref{lem:walks}(iii). The tail bound
$\Pr[Z\ge u]\le e^{-u^2/2}$ for $Z\sim\mathcal N(0,1)$ and $u\ge0$ gives
\[
 \Pr\!\left[\arg\max_{v\in\{\cand1,\cand2\}}\langle\emb{v},\vz_1\rangle=\ansv\right]
 \ge1-\exp\!\left(-\frac{\mgn^2}{4\sigma^2e^{2\gain\sqrt{\nedge}}}\right).
\]
The component of $\nse$ orthogonal to the vertex span is frozen and invisible to
the readout, and $\mathbb E[\vz_t]=\emb{\tokA}+e^{\gain t\Prop}\emb{\rootnode}$.
\end{proof}

\subsection{Readout over the full vocabulary}
\label{app:readout}

\begin{remark}[Full-vocabulary readout]
\label{rem:decoding}
Under the noiseless initialization of \Cref{thm:existence},
a reachable vertex other than $\ansv$ can have a larger
coefficient because the coefficients count walks. We use the
candidate embeddings stored in buffers~3 and~4 to define
\[
  S_v =
  \langle\emb{v},\vz_1\rangle
  + M_0\bigl(
    \langle\emb{v},\emb{\cand1}\rangle
    +\langle\emb{v},\emb{\cand2}\rangle
  \bigr).
  \qquad v\in\vocab.
\]
By orthonormality, the added term is $M_0$ on the two
candidates and zero elsewhere. Every vertex coefficient is
at most $e^{\gain\nmax}$; the answer-token coefficient is
one, and all other special-token coefficients are zero.
Thus, choosing $M_0=2e^{\gain\nmax}$ gives
$S_{\ansv}\ge M_0+\mgn$, $S_{\wrongv}=M_0$, and
$S_v<M_0$ for every other token.

The full-vocabulary readout therefore returns $\ansv$ with
margin at least $\mgn$. This is a fixed linear decoder on
the content slot and candidate buffers, with $M_0$ depending
only on $\nmax$ and $\gain$. The velocity field is unchanged.
\end{remark}

\section{Sphere Retraction}
\label{app:sphere}

\subsection{Local geometric effect}
\label{app:psi-angular-flow}

For a nonzero position vector, define
$\psi(z)=Rz/\lVert z\rVert_2$, with $R=\sqrt d$.
The following statement describes the local effect of this
retraction.

\begin{proposition}[Local effect of sphere retraction]
\label{prop:angular}
Let $v(z,t)$ be continuously differentiable near the sphere
$\lVert z\rVert_2=R$ over a fixed interval $[0,T]$, with $T<1$.
For $z$ on the sphere and sufficiently small $\Delta t$,
\[
  \psi\bigl(z+\Delta t\,v(z,t)\bigr)
  =
  z+\Delta t\,P_{\hat z}v(z,t)
  +\mathcal O(\Delta t^2),
  \qquad
  P_{\hat z}=\mI-\hat z\hat z^\top,
  \quad
  \hat z=z/R.
\]
The associated velocity is tangent to the sphere.
\end{proposition}

\begin{proof}
For nonzero $z$,
\[
  D\psi(z)
  =
  \frac{R}{\lVert z\rVert_2}
  \left(\mI-\frac{zz^\top}{\lVert z\rVert_2^2}\right).
\]
Taylor expansion at a point with $\lVert z\rVert_2=R$
gives the stated formula. Since
$\langle z,P_{\hat z}v\rangle=0$, the resulting
continuous-time dynamics preserve the norm.
\end{proof}

\paragraph{Relation to the implemented update.}
The algorithms in \Cref{app:long-rollout} evaluate the
denoiser at the raw state and use its retracted version
as the update base. If their difference is
$\mathcal O(h)$, where $h$ is the largest grid spacing,
local Lipschitz continuity of the denoiser makes the
prediction difference $\mathcal O(h)$.
On an interval bounded away from $t=1$, this changes an
update by $\mathcal O(\Delta t_k h)$.
The proposition therefore describes the leading-order
interior dynamics under these conditions. It does not
establish convergence through the endpoint $t=1$.

\subsection{Retraction ablation and norm diagnostic}

We trained two Sudoku-Extreme models with and without
$\psi$, using the same time-conditioned,
constant-learning-rate recipe. At $\Neval=64$, their
best recorded test accuracies were $67\%$ with retraction
and $55\%$ without it. Both models were trained separately.
The ablation uses a different recipe from the
cosine-schedule runs in the main comparison.

In the norm diagnostic (\Cref{fig:norm_growth}), evaluated over a $252$-step rollout, 
the model trained without
retraction reaches a latent norm roughly $44{,}000$
times its initial scale. Retraction limits this growth,
but does not hold the raw state at a fixed norm.
The fixed radius applies to the retracted update base.

The construction in \Cref{sec:theory} also allows the latent norm
to grow during propagation, while its final candidate decision is
unchanged by positive rescaling.

\begin{figure}[t]
  \centering
  \includegraphics[width=0.75\linewidth]{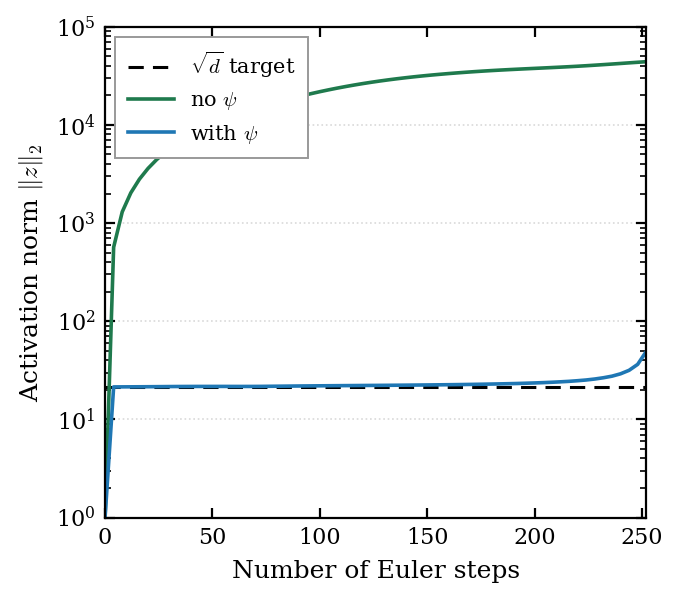}
  \caption{Raw latent norms in separately trained
  Sudoku-Extreme models with and without retraction.
  The dashed line marks $\sqrt{448}$, the norm of the
  retracted update base. The raw state itself need not
  remain on this sphere.}
  \label{fig:norm_growth}
\end{figure}

\section{Long-Rollout Training and Inference}
\label{app:long-rollout}

This appendix gives the full procedure used for the long Sudoku and Maze
rollouts. Sphere retraction is applied independently to each answer position:
\[
  \psi(\vz^{(i)})=
  \sqrt{d}\,\frac{\vz^{(i)}}{\lVert \vz^{(i)}\rVert_2},
  \qquad i=1,\ldots,L.
\]

The denoiser reads the raw state. After the first step,
retraction is applied only to the update base. Noise scales
and label smoothing are given in
\Cref{tab:ufm-hyperparameters}.

\begin{algorithm}[t]
  \caption{Long-rollout UFM training on one batch with sphere retraction}
  \label{alg:ufm-long-rollout}
  \begin{algorithmic}[1]
    \REQUIRE Batch $(c,\vy)$, parameters $\theta$,
    rollout length $\Ntrain$,
    backpropagation length $1\le\Nback\le\Ntrain$,
    time-grid distribution $q$, noise scale $\sigma$,
    and retraction $\psi$

    \STATE Sample noise $\nse\sim\mathcal{N}(0,\mI)$
    in the latent answer space
    \STATE Sample $0\le t_0<t_1<\cdots<t_{\Ntrain}=1$
    from $q$
    \STATE $\vz_{t_0}\gets
    (1-t_0)\sigma\nse+t_0\operatorname{Embed}(\vy)$

    \FOR{$k=0$ \TO $\Ntrain-1$}
      \STATE Evaluate this update without gradient tracking
      when $k<\Ntrain-\Nback$
      \STATE $\Delta t_k\gets t_{k+1}-t_k$
      \STATE $\dnz^{(k)}\gets\dnz(\vz_{t_k},t_k,c)$

      \IF{$t_{k+1}=1$}
        \STATE $\vz_{t_{k+1}}\gets\dnz^{(k)}$
        \COMMENT{final endpoint prediction}
      \ELSE
        \STATE $y_{t_k}\gets\vz_{t_k}$
        \IF{$k>0$}
          \STATE $y_{t_k}\gets\psi(\vz_{t_k})$
        \ENDIF
        \STATE $\vfield^{(k)}\gets
        (\dnz^{(k)}-y_{t_k})/(1-t_k)$
        \STATE $\vz_{t_{k+1}}\gets
        y_{t_k}+\Delta t_k\,\vfield^{(k)}$
      \ENDIF

      \IF{$k<\Ntrain-\Nback$}
        \STATE $\vz_{t_{k+1}}\gets
        \operatorname{sg}(\vz_{t_{k+1}})$
      \ENDIF
    \ENDFOR

    \STATE $\hat{\vy}\gets\operatorname{StableMax}
    \bigl(\operatorname{RMSNorm}
    (\vz_{t_{\Ntrain}})\Wdec^\top\bigr)$
    \COMMENT{decode once}
    \STATE $\mathcal{L}\gets\operatorname{CE}(\hat{\vy},\vy)$
    \COMMENT{with the specified label smoothing}
    \RETURN $\mathcal{L},\nabla_\theta\mathcal{L}$
  \end{algorithmic}
\end{algorithm}

\begin{algorithm}[t]
  \caption{UFM inference with sphere retraction}
  \label{alg:ufm-inference}
  \begin{algorithmic}[1]
    \REQUIRE Prompt $c$, EMA parameters $\theta^{\mathrm{EMA}}$,
    evaluation grid $0=t_0<\cdots<t_{\Neval}=1$,
    noise scale $\sigma$, and retraction $\psi$

    \STATE Use $\theta^{\mathrm{EMA}}$ for the denoiser
    and decoder
    \STATE Sample $\nse\sim\mathcal{N}(0,\mI)$
    in the latent answer space
    \STATE $\vz_{t_0}\gets\sigma\nse$

    \FOR{$k=0$ \TO $\Neval-1$}
      \STATE $\Delta t_k\gets t_{k+1}-t_k$
      \STATE $\dnz^{(k)}\gets\dnz(\vz_{t_k},t_k,c)$

      \IF{$t_{k+1}=1$}
        \STATE $\vz_{t_{k+1}}\gets\dnz^{(k)}$
        \COMMENT{final endpoint prediction}
      \ELSE
        \STATE $y_{t_k}\gets\vz_{t_k}$
        \IF{$k>0$}
          \STATE $y_{t_k}\gets\psi(\vz_{t_k})$
        \ENDIF
        \STATE $\vfield^{(k)}\gets
        (\dnz^{(k)}-y_{t_k})/(1-t_k)$
        \STATE $\vz_{t_{k+1}}\gets
        y_{t_k}+\Delta t_k\,\vfield^{(k)}$
      \ENDIF
    \ENDFOR

    \STATE $\hat{\vy}\gets\operatorname{StableMax}
    \bigl(\operatorname{RMSNorm}
    (\vz_{t_{\Neval}})\Wdec^\top\bigr)$
    \RETURN The highest-probability token at each position
  \end{algorithmic}
\end{algorithm}

\section{Experimental Details}
\label{app:exp}

This appendix describes the data splits, training settings, and
evaluation procedures used in our experiments.

\paragraph{Reporting convention.}
For the single-rollout results in \Cref{tab:ufm-pass1},
each noise seed gives one independently initialized rollout per
test example, using a fixed checkpoint. We report the mean and
standard deviation of exact-match accuracy over five noise seeds.
Values are rounded to one decimal place, so a displayed standard
deviation of $0.0$ may be nonzero before rounding.
The reported checkpoints have the DiT's time input disabled.

\begin{table}[t]
  \centering
  \caption{UFM single-rollout accuracy with $\Neval=128$ and a
  uniform time grid. Results are percentages, reported as mean
  $\pm$ standard deviation over five noise seeds.}
  \label{tab:ufm-pass1}
  \begin{tabular}{@{}lrr@{}}
    \toprule
    Benchmark & Test examples & Pass@1 (\%) \\
    \midrule
    Sudoku-Extreme & $422{,}786$ & $74.4\pm0.0$ \\
    Sudoku-Hard   & $2{,}000$   & $86.9\pm0.9$ \\
    Maze-Hard     & $1{,}000$   & $89.3\pm0.5$ \\
    \bottomrule
  \end{tabular}
\end{table}

\subsection{Selecting among rollouts}
\label{app:chooser}

We evaluate $K=100$ rollouts per problem on a $2000$-puzzle subset of the Sudoku-Extreme test set and on all $1000$ Maze-Hard test mazes.
On the Sudoku subset, Pass@1 is $74.2\%$, Pass@100 is
$98.65\%$, and margin-based selection reaches $98.6\%$.
On Maze-Hard, margin-based selection reaches $92.0\%$,
compared with $95.0\%$ Pass@100.

We select among the $K$ rollouts using the mean gap between
the largest and second-largest raw logits,
\[
  \widehat{\mgn}
  =
  \frac{1}{L}\sum_{i=1}^{L}
  \bigl(\ell_i^{(1)}-\ell_i^{(2)}\bigr).
\]
Here, $\ell_i^{(1)}$ and $\ell_i^{(2)}$ are the two largest
logits at position $i$, before probability normalization.
This parameter-free score is an output-space analogue of
the separation margin in \Cref{thm:existence}. 

Selection is less effective on Maze-Hard. A Sudoku answer
has $81$ positions, whereas a Maze answer contains $900$
cells describing the maze layout and solution path.
Averaging scores over cells shared by different candidate
answers can dilute differences in their paths.

On a probe of $300$ Maze boards with $50$ rollouts each at $\Neval{=}128$, a subset of the boards in the evaluation above, $18$ boards are recoverable, meaning that the first rollout is wrong but at least one of the $50$ is correct; the remaining $282$ boards are either already solved by the first rollout or solved by none of the $50$. On these $18$ boards, correct and incorrect rollouts differ in mean margin by $0.07$, for mean margins around $200$ in logit units, and the chooser picks a correct rollout on $9$ of $18$, against $38\%$ for a uniformly random pick among the distinct answers.

These diagnostics suggest that the mean logit margin is less
reliable for selecting long maze answers and a learned scorer, such as the latent reward model of
GRAM \citep{gram}, or the halting mechanism used in \citet{JolicoeurMartineau2025,wang2025hrm} is a natural replacement.

\subsection{ProsQA}
FLM and UFM use a DiT backbone; S-FLM keeps the architecture of \citet{HypersphericalFlows2026}, which
was designed for that method. All models have two layers and approximately $15.7$M
parameters and are trained on the same data. As in \citet{zhu2025reasoning}, we use the ProsQA questions whose answer is at most four hops from the root: $14{,}785$ training, $257$ validation, and $419$ test questions, out of the $17{,}886$ / $300$ / $500$ released. UFM is trained with $\Ntrain=5$
steps, so each example costs it $5$ forward and backward passes. To approximately match the number of forward and backward passes, FLM and S-FLM see $5\times$ more training examples: $43{,}500$ optimizer steps
at batch size $512$ (${\approx}22.3$M examples), against $34{,}500$ steps at batch
size $128$ for UFM (${\approx}4.4$M examples). UFM uses a smaller batch to fit the rollout in memory. Gradients flow through all five updates, but the gradient passed back through the latent state is scaled by $0.7$ at each update; the forward value is unchanged.


\subsection{Comparison with Recursive Reasoners}
We also place UFM in context against published recursive reasoners. These rows are reported from the respective papers rather than reruns, so they provide a reference point rather than a controlled pairwise comparison.

\begin{table}[t]
  \centering
  \small
  \caption{Test accuracy (\%) on Sudoku-Extreme and Maze-Hard with a single
inference trajectory (Pass@1). Recursive-reasoner numbers are as reported in
the respective papers; we did not rerun them. FLM and S-FLM are our
runs of the authors' released code with the default greedy final step, at their
best checkpoint and inference setting. UFM uses $\Neval=128$ on a uniform grid; its results are means
and standard deviations over five noise seeds, evaluated on
$422{,}786$ Sudoku-Extreme puzzles and $1000$ Maze-Hard mazes. FRM's entry is its reported peak on a fixed $1000$-puzzle
Sudoku-Extreme test subset. Its maze evaluation uses
Maze-Unique, so no Maze-Hard result is listed.}

  \label{tab:main-single}
  \begin{tabular}{@{}lrrr@{}}
    \toprule
    Model & Params & Sudoku-Extreme & Maze-Hard \\
    \midrule
    \multicolumn{4}{@{}l}{\textit{Recursive reasoners}} \\
    \quad Attractor Model   \citep{attractor}              & 7M    & 54.3\%  & 46.7\% \\
    \quad HRM  \citep{wang2025hrm}         & 27M   & 55.0\%  & 74.5\% \\
    \quad URM   \citep{urm}                 & 14M   & 77.6\%  & --     \\
    \quad TRM$^{a}$ \citep{JolicoeurMartineau2025}                      & 5M/7M & 84.1\%  & 85.3\% \\
    \quad Attractor Model$^{b}$  \citep{attractor}         & 27M   & 91.4\%  & 93.1\% \\
    \quad EqR ($D{=}64$)$^{c}$ \citep{eqr}         & 5M    & 93.0\%  & --     \\
    \quad FPRM  \citep{fprm}                          & 7M    & 94.2\%  & 87.0\% \\
    \addlinespace[2pt]
    \midrule
    \multicolumn{4}{@{}l}{\textit{Flow models}} \\
    \quad FLM \citep{FlowMapLM2026}                        & 28.6M$^{\dagger}$ & 10.7\%  & 40.4\% \\
\quad S-FLM \citep{HypersphericalFlows2026} & 28.6M$^{\dagger}$ & \phantom{0}9.4\% & 49.9\% \\ 
\quad UFM (ours) & 8.4M & $74.4\pm0.0\%$ & $89.3\pm0.5\%$\\
\quad FRM \citep{helbling2026frm} & 7M & 99.5\%$^{*}$ & -- \\ 

    \bottomrule
  \end{tabular}
  \par\smallskip\raggedright\footnotesize
  $^{a}$TRM-MLP (5M) on Sudoku, reproduced value from \citet{eqr}; the original
  paper reports 87.4\%. TRM-Att (7M) on Maze-Hard as reported.
  $^{b}$Not reproduced by \citet{fprm}, who obtain 71.4\% on Sudoku.
  $^{c}$$D{=}64$ outer iterations at test time, $4\times$ the training budget;
  at $D{=}16$ the result is 86.4\%.
  $^{\dagger}$At $8$M, both baselines failed to train stably, reaching $1.5\%$ on
Sudoku-Extreme and then declining. $^{*}$FRM reports this peak on a fixed $1000$-puzzle
\emph{subset} of the Sudoku-Extreme test set, not the full
$422{,}786$-puzzle test set used for our Pass@1 evaluation.

\end{table}

\begin{table}[t]
  \centering
  \small
  \caption{Test accuracy (\%) on Sudoku-Extreme and Maze-Hard with multiple
  inference rollouts. \emph{Rollouts} is the number of trajectories drawn at
  inference, followed by selection or voting. Recursive-reasoner numbers are as
  reported in the respective papers. UFM uses $N_{\mathrm{eval}}{=}128$ Euler
  steps and $K{=}100$ rollouts with the chooser of \Cref{sec:chooser}. UFM is evaluated on a $2000$-puzzle Sudoku-Extreme test subset
and all $1000$ Maze-Hard test mazes (\Cref{app:chooser}).}
  \label{tab:main-multi}
  \begin{tabular}{@{}lrrrr@{}}
    \toprule
    Model & Params & Rollouts & Sudoku-Extreme & Maze-Hard \\
    \midrule
    \multicolumn{5}{@{}l}{\textit{Recursive reasoners}} \\
    \quad GRAM$^{a}$  \citep{gram}                    & 10M   & 20  & 97.0\%  & --     \\
    \quad PTRM$^{b}$  \citep{ptrm}                    & 5M/7M & 100 & 98.75\% & 86.73\% \\
    \quad EqR ($D{=}64$)$^{c,d}$  \citep{eqr}        & 5M    & 128 & 99.8\%  & --     \\
    \addlinespace[2pt]
    \midrule
    \multicolumn{5}{@{}l}{\textit{Flow models}} \\
    \quad UFM (ours) & 8.4M & 100 & 98.6\% & 92.0\% \\
    \bottomrule
  \end{tabular}
  \par\smallskip\raggedright\footnotesize
  $^{a}$$20$ parallel trajectories selected by a learned latent reward model.
  $^{b}$Inference-only on a TRM checkpoint; best-Q@100 on Sudoku, majority vote
  on Maze; Pass@100 is 99.06\% / 95.63\%.
  $^{c}$$D{=}64$ outer iterations at test time, $4\times$ the training budget.
  $^{d}$EqR reports 93.0\% on Maze-Unique, a filtered variant, not Maze-Hard.
\end{table}

\subsection{Baseline replication on Sudoku-Hard}
\label{app:sflm-replication}
We use the best tested checkpoint of each baseline at $30$k
training steps, corresponding to $1.6$ hours for FLM and
$1.7$ hours for S-FLM. For S-FLM, we used the truncated schedule in the authors' released code (the $43.2\%$ setting in Table~1 of \citet{HypersphericalFlows2026}), not the adaptive-schedule variant.

The released split has no validation set, so each baseline checkpoint is selected on
the same $2$k puzzles on which it is reported, which favors the baselines. The
UFM checkpoint is instead selected on a separate $2$k validation split (best
exact match at $N{=}64$); this is the final checkpoint ($19.7$k steps,
$1.1$\,h). The updated evaluation reaches $86.9\pm0.9\%$ at $N=128$
over five noise seeds. \Cref{tab:acc-vs-n} reports
accuracy as a function of the number of sampling steps, with the default greedy
final step. With ancestral noise removal at $N{=}128$, S-FLM reaches
$50.9\%$, the number in \Cref{tab:sflm_hard}; FLM has no such option. Training
times exclude evaluation.

\begin{table}[h]
\centering\small
\caption{Accuracy (\%) on the hard Sudoku split of \citet{HypersphericalFlows2026}
as a function of the number of sampling steps $N$, with the default greedy final
step and without ancestral noise removal. FLM and S-FLM are our runs
of the authors' released code ($28.6$M parameters) with the cosine schedule of
their code base, at their best checkpoint ($30$k steps), selected on these same
puzzles, which favors the baselines. With ancestral noise removal at $N{=}128$,
S-FLM reaches $50.9\%$ (\Cref{tab:sflm_hard}); with the released
constant schedule it reaches $36.8\%$, and its authors report $43.2\%$ with the truncated linear noise schedule and $45.0\%$ with the additional adaptive schedule. UFM
($8.4$M parameters) is trained with $24$-step rollouts and selected on a
separate validation split. Training is training-only wall-clock time on one A100. The UFM row reports an earlier single-run step-count sweep.
The updated five-noise-seed result at $N=128$ is
$86.9\pm0.9\%$ (\Cref{tab:ufm-pass1}). The second UFM row uses the same setup with a longer training
budget of $22$ hours. We report it separately from the main
comparison in \Cref{tab:sflm_hard}.}
\label{tab:acc-vs-n}
\begin{tabular}{llrrrrr}
\toprule
Model & Train (h) & $N{=}8$ & $16$ & $32$ & $64$ & $128$ \\
\midrule
FLM              & 1.6\,h   & 51.2 & 51.8 & 51.8 & 51.6 & 51.9 \\
S-FLM & 1.7\,h   & 48.9 & 49.5 & 50.0 & 50.2 & 50.1 \\
\midrule
UFM (ours)       & 1.1\,h   & 2.8  & 65.1 & 80.2 & 84.4 & 86.1 \\
UFM (ours)       & 22\,h     & 64.1 & 97.0 & 99.4 & 99.8 & 99.8 \\
\bottomrule

\end{tabular}
\end{table}

\subsection{Hyperparameters}
\label{app:ufm-hyperparameters}

\paragraph{UFM on Sudoku and Maze.}
All three benchmarks use a two-layer DiT with hidden width $448$, eight
attention heads, and $8.4$M parameters. The backbone uses pre-norm RMSNorm,
SwiGLU with expansion factor four, adaLN-Zero, two-dimensional rotary
position embeddings, and $16$ puzzle prefix tokens. The implementation
supports sinusoidal time conditioning with frequency dimension $256$.
The reported checkpoints disable this input; flow time still enters through
the update coefficients.

We use AdamW with $\beta=(0.9,0.95)$, gradient clipping at $1.0$, dropout
$0.1$, and an exponential moving average of the weights with decay $0.999$.
After warmup, the learning rate follows a cosine schedule from $10^{-4}$
to $10^{-5}$. We retain weights and accumulation in FP32 and use BF16
matrix multiplications. Settings that differ across benchmarks are
listed in \Cref{tab:ufm-hyperparameters}.

\begin{table}[t]
  \centering
  \caption{Benchmark-specific UFM training settings. The noise scale
  $\sigma$ is the standard deviation of each latent noise coordinate.
  The Sudoku-Hard column describes the run used for the main comparison.}
  \label{tab:ufm-hyperparameters}
  \begin{tabular}{@{}lccc@{}}
    \toprule
    Setting & Sudoku-Extreme & Maze-Hard & Sudoku-Hard \\
    \midrule
    Global batch size       & $768$ & $128$ & $256$ \\
    Noise scale $\sigma$    & $1/\sqrt{448}$ & $1$ & $1/\sqrt{448}$ \\
    Label smoothing         & $0.1$ & $0$ & $0.1$ \\
    Weight decay            & $0.75$ & $1.0$ & $0.75$ \\
    Warmup steps            & $2{,}000$ & $2{,}000$ & $500$ \\
    Training budget (steps) & $130$k & $125$k & $19.7$k \\
    \bottomrule
  \end{tabular}
\end{table}

\paragraph{Training grid and terminal loss.}
We train with $\Ntrain=24$ updates and backpropagate through the final
$\Nback=6$; the preceding updates run without gradient tracking.
For each example, we sample
$t_0\sim\mathcal{U}[0,t_{0,\max}]$ and initialize the state from the
noise--answer interpolant. The upper bound $t_{0,\max}$ decreases from
$0.6$ to $0.2$ over the first $30\%$ of training. We sample and sort
$23$ further times in $(t_0,1)$ independently for each example.

The final network call returns the predicted endpoint directly,
equivalent to taking the last Euler update to $t=1$. The decoder applies
RMS normalization followed by an untied linear vocabulary projection.
We compute one StableMax cross-entropy loss at this endpoint, with the
label smoothing shown in \Cref{tab:ufm-hyperparameters}.

\subsection{Computing Resources}

UFM training uses one NVIDIA A100 on Sudoku-Extreme and Sudoku-Hard,
and two A100s on Maze-Hard. Training time excludes evaluation and is
estimated from the median steady-state iteration rate in the logs.

On Sudoku-Hard, UFM takes $1.14$ hours for $19.7$k optimization steps.
FLM and S-FLM take $1.07$ and $1.12$ hours, respectively, for $20$k
steps. Their selected $30$k-step checkpoints require $1.61$ and
$1.68$ hours of training. These estimates give the rounded training
times reported in the main comparison.
\end{document}

%% file: math_commands.tex
\usepackage{amsmath,amsfonts,bm}

\def\eqref#1{equation~\ref{#1}}

\def\1{\bm{1}}

\def\vh{{\bm{h}}}

\def\vp{{\bm{p}}}

\def\vw{{\bm{w}}}
\def\vx{{\bm{x}}}
\def\vy{{\bm{y}}}
\def\vz{{\bm{z}}}

\def\mI{{\bm{I}}}

\def\mW{{\bm{W}}}

\DeclareMathAlphabet{\mathsfit}{\encodingdefault}{\sfdefault}{m}{sl}
\SetMathAlphabet{\mathsfit}{bold}{\encodingdefault}{\sfdefault}{bx}{n}



%% file: iclr2027_conference.bib
@misc{wang2025hrm,
      title={Hierarchical Reasoning Model}, 
      author={Guan Wang and Jin Li and Yuhao Sun and Xing Chen and Changling Liu and Yue Wu and Meng Lu and Sen Song and Yasin Abbasi Yadkori},
      year={2025},
      eprint={2506.21734},
      archivePrefix={arXiv},
      primaryClass={cs.AI},
      url={https://arxiv.org/abs/2506.21734}, 
}

@misc{eqr,
      title={Equilibrium Reasoners: Learning Attractors Enables Scalable Reasoning}, 
      author={Benhao Huang and Zhengyang Geng and Zico Kolter},
      year={2026},
      eprint={2605.21488},
      archivePrefix={arXiv},
      primaryClass={cs.LG},
      url={https://arxiv.org/abs/2605.21488}, 
}

@misc{gram,
      title={Generative Recursive Reasoning}, 
      author={Junyeob Baek and Mingyu Jo and Minsu Kim and Mengye Ren and Yoshua Bengio and Sungjin Ahn},
      year={2026},
      eprint={2605.19376},
      archivePrefix={arXiv},
      primaryClass={cs.AI},
      url={https://arxiv.org/abs/2605.19376}, 
}

@misc{attractor,
      title={Solve the Loop: Attractor Models for Language and Reasoning}, 
      author={Jacob Fein-Ashley and Paria Rashidinejad},
      year={2026},
      eprint={2605.12466},
      archivePrefix={arXiv},
      primaryClass={cs.LG},
      url={https://arxiv.org/abs/2605.12466}, 
}

@misc{urm,
      title={Universal Reasoning Model}, 
      author={Zitian Gao and Lynx Chen and Yihao Xiao and He Xing and Ran Tao and Haoming Luo and Joey Zhou and Bryan Dai},
      year={2025},
      eprint={2512.14693},
      archivePrefix={arXiv},
      primaryClass={cs.AI},
      url={https://arxiv.org/abs/2512.14693}, 
}

@misc{suleymanzade2026thinkingloopedflows,
      title={Thinking with Looped Flows}, 
      author={Ayhan Suleymanzade and Chanhyuk Lee and Floor Eijkelboom and Nicholas M. Boffi and İsmail İlkan Ceylan and Jinwoo Kim},
      year={2026},
      eprint={2609.11801},
      archivePrefix={arXiv},
      primaryClass={cs.LG},
      url={https://arxiv.org/abs/2609.11801}, 
}

@misc{helbling2026frm,
      title={Flow Reasoning Models: Turning Flows Into Efficient Recurrent Reasoners}, 
      author={Alec Helbling and Andrey Bryutkin and Mauro Martino and Duen Horng Chau and Nima Dehmamy and Hendrik Strobelt},
      year={2026},
      eprint={2606.29150},
      archivePrefix={arXiv},
      primaryClass={cs.AI},
      url={https://arxiv.org/abs/2606.29150}, 
}

@misc{JolicoeurMartineau2025,
      title={Less is More: Recursive Reasoning with Tiny Networks}, 
      author={Alexia Jolicoeur-Martineau},
      year={2025},
      eprint={2510.04871},
      archivePrefix={arXiv},
      primaryClass={cs.LG},
      url={https://arxiv.org/abs/2510.04871}, 
}

@misc{FlowMapLM2026,
      title={Flow Map Language Models: One-step Language Modeling via Continuous Denoising}, 
      author={Chanhyuk Lee and Jaehoon Yoo and Manan Agarwal and Sheel Shah and Jerry Huang and Aditi Raghunathan and Seunghoon Hong and Nicholas M. Boffi and Jinwoo Kim},
      year={2026},
      eprint={2602.16813},
      archivePrefix={arXiv},
      primaryClass={cs.CL},
      url={https://arxiv.org/abs/2602.16813}, 
}

@misc{CategoricalFlowMaps2026,
      title={Categorical Flow Maps}, 
      author={Daan Roos and Oscar Davis and Floor Eijkelboom and Michael Bronstein and Max Welling and İsmail İlkan Ceylan and Luca Ambrogioni and Jan-Willem van de Meent},
      year={2026},
      eprint={2602.12233},
      archivePrefix={arXiv},
      primaryClass={cs.LG},
      url={https://arxiv.org/abs/2602.12233}, 
}

@misc{ELF2026,
      title={{ELF}: Embedded Language Flows}, 
      author={Keya Hu and Linlu Qiu and Yiyang Lu and Hanhong Zhao and Tianhong Li and Yoon Kim and Jacob Andreas and Kaiming He},
      year={2026},
      eprint={2605.10938},
      archivePrefix={arXiv},
      primaryClass={cs.CL},
      url={https://arxiv.org/abs/2605.10938}, 
}

@misc{ptrm,
      title={Probabilistic Tiny Recursive Model}, 
      author={Amin Sghaier and Ali Parviz and Alexia Jolicoeur-Martineau},
      year={2026},
      eprint={2605.19943},
      archivePrefix={arXiv},
      primaryClass={cs.AI},
      url={https://arxiv.org/abs/2605.19943}, 
}

@misc{HypersphericalFlows2026,
      title={Language Modeling with Hyperspherical Flows}, 
      author={Justin Deschenaux and Caglar Gulcehre},
      year={2026},
      eprint={2605.11125},
      archivePrefix={arXiv},
      primaryClass={cs.LG},
      url={https://arxiv.org/abs/2605.11125}, 
}

@misc{fprm,
      title={Fixed-Point Reasoners: Stable and Adaptive Deep Looped Transformers}, 
      author={Sajad Movahedi and Vera Milovanović and Shlomo Libo Feigin and Alexander Theus and Thomas Hofmann and Valentina Boeva and T. Konstantin Rusch and Antonio Orvieto},
      year={2026},
      eprint={2606.18206},
      archivePrefix={arXiv},
      primaryClass={cs.AI},
      url={https://arxiv.org/abs/2606.18206}, 
}

@misc{SelfConditionedFlowMap2026,
      title={Self-conditioned Flow Map Language Models via Fixed-point Flows}, 
      author={Jaehoon Yoo and Wonjung Kim and Floor Eijkelboom and Chanhyuk Lee and Nicholas M. Boffi and Seunghoon Hong and Jinwoo Kim},
      year={2026},
      eprint={2607.00714},
      archivePrefix={arXiv},
      primaryClass={cs.CL},
      url={https://arxiv.org/abs/2607.00714}, 
}

@misc{ScalingCategoricalFlowMaps2026,
      title={Scaling Categorical Flow Maps}, 
      author={Oscar Davis and Anastasiia Filippova and Pierre Ablin and Victor Turrisi and Amitis Shidani and Marco Cuturi and Louis Béthune},
      year={2026},
      eprint={2605.07820},
      archivePrefix={arXiv},
      primaryClass={cs.LG},
      url={https://arxiv.org/abs/2605.07820}, 
}

@misc{zhu2025reasoning,
      title={Reasoning by Superposition: A Theoretical Perspective on Chain of Continuous Thought}, 
      author={Hanlin Zhu and Shibo Hao and Zhiting Hu and Jiantao Jiao and Stuart Russell and Yuandong Tian},
      year={2025},
      eprint={2505.12514},
      archivePrefix={arXiv},
      primaryClass={cs.LG},
      url={https://arxiv.org/abs/2505.12514}, 
}

@article{Albergo2023StochasticInterpolants,
  author  = {Michael Albergo and Nicholas M. Boffi and Eric Vanden-Eijnden},
  title   = {Stochastic Interpolants: A Unifying Framework for Flows and Diffusions},
  journal = {Journal of Machine Learning Research},
  year    = {2025},
  volume  = {26},
  number  = {209},
  pages   = {1--80},
  url     = {http://jmlr.org/papers/v26/23-1605.html}
}

@misc{prieto2025stablemax,
      title={Grokking at the Edge of Numerical Stability}, 
      author={Lucas Prieto and Melih Barsbey and Pedro A. M. Mediano and Tolga Birdal},
      year={2025},
      eprint={2501.04697},
      archivePrefix={arXiv},
      primaryClass={cs.LG},
      url={https://arxiv.org/abs/2501.04697}, 
}

@misc{snell2024testtime,
      title={Scaling {LLM} Test-Time Compute Optimally can be More Effective than Scaling Model Parameters}, 
      author={Charlie Snell and Jaehoon Lee and Kelvin Xu and Aviral Kumar},
      year={2024},
      eprint={2408.03314},
      archivePrefix={arXiv},
      primaryClass={cs.LG},
      url={https://arxiv.org/abs/2408.03314}, 
}

@inproceedings{
Lipman2023FlowMatching,
title={Flow Matching for Generative Modeling},
author={Yaron Lipman and Ricky T. Q. Chen and Heli Ben-Hamu and Maximilian Nickel and Matthew Le},
booktitle={The Eleventh International Conference on Learning Representations },
year={2023},
url={https://openreview.net/forum?id=PqvMRDCJT9t}
}

@inproceedings{
Liu2023RectifiedFlow,
title={Flow Straight and Fast: Learning to Generate and Transfer Data with Rectified Flow},
author={Xingchao Liu and Chengyue Gong and Qiang Liu},
booktitle={The Eleventh International Conference on Learning Representations },
year={2023},
url={https://openreview.net/forum?id=XVjTT1nw5z}
}

@inproceedings{Peebles2023DiT,
  author={Peebles, William and Xie, Saining},
  booktitle={2023 IEEE/CVF International Conference on Computer Vision (ICCV)}, 
  title={Scalable Diffusion Models with Transformers}, 
  year={2023},
  volume={},
  number={},
  pages={4172-4182},
  doi={10.1109/ICCV51070.2023.00387}}

@inproceedings{
clark2024draft,
title={Directly Fine-Tuning Diffusion Models on Differentiable Rewards},
author={Kevin Clark and Paul Vicol and Kevin Swersky and David J. Fleet},
booktitle={The Twelfth International Conference on Learning Representations},
year={2024},
url={https://openreview.net/forum?id=1vmSEVL19f}
}

@inproceedings{eijkelboom2024variational,
 author = {Eijkelboom, Floor and Bartosh, Grigory and Naesseth, Christian A. and Welling, Max and van de Meent, Jan-Willem},
 booktitle = {Advances in Neural Information Processing Systems},
 doi = {10.52202/079017-0374},
 editor = {A. Globerson and L. Mackey and D. Belgrave and A. Fan and U. Paquet and J. Tomczak and C. Zhang},
 pages = {11735--11764},
 publisher = {Curran Associates, Inc.},
 title = {Variational Flow Matching for Graph Generation},
 url = {https://proceedings.neurips.cc/paper_files/paper/2024/file/15b780350b302a1bf9a3bd273f5c15a4-Paper-Conference.pdf},
 volume = {37},
 year = {2024}
}

@misc{Hao2024Coconut,
      title={Training Large Language Models to Reason in a Continuous Latent Space}, 
      author={Shibo Hao and Sainbayar Sukhbaatar and DiJia Su and Xian Li and Zhiting Hu and Jason Weston and Yuandong Tian},
      year={2026},
      eprint={2412.06769},
      archivePrefix={arXiv},
      primaryClass={cs.CL},
      url={https://arxiv.org/abs/2412.06769}, 
}

@inproceedings{Vaswani2017Attention,
title	= {Attention is All You Need},
author	= {Ashish Vaswani and Noam Shazeer and Niki Parmar and Jakob Uszkoreit and Llion Jones and Aidan N. Gomez and Lukasz Kaiser and Illia Polosukhin},
booktitle = {Advances in Neural Information Processing Systems},
year	= {2017},
URL	= {https://arxiv.org/pdf/1706.03762.pdf}}

@inproceedings{
wei2022cot,
title={Chain of Thought Prompting Elicits Reasoning in Large Language Models},
author={Jason Wei and Xuezhi Wang and Dale Schuurmans and Maarten Bosma and Brian Ichter and Fei Xia and Ed H. Chi and Quoc V Le and Denny Zhou},
booktitle={Advances in Neural Information Processing Systems},
editor={Alice H. Oh and Alekh Agarwal and Danielle Belgrave and Kyunghyun Cho},
year={2022},
url={https://openreview.net/forum?id=_VjQlMeSB_J}
}

@inproceedings{katharopoulos2020transformers,
  author    = {Katharopoulos, Angelos and Vyas, Apoorv and Pappas, Nikolaos and Fleuret, Fran{\c{c}}ois},
  title     = {Transformers Are {RNN}s: Fast Autoregressive Transformers with Linear Attention},
  booktitle = {Proceedings of the 37th International Conference on Machine Learning},
  series    = {Proceedings of Machine Learning Research},
  volume    = {119},
  pages     = {5156--5165},
  publisher = {PMLR},
  year      = {2020},
  url       = {https://proceedings.mlr.press/v119/katharopoulos20a.html}
}
